\documentclass[10pt,journal,compsoc]{IEEEtran}
\ifCLASSOPTIONcompsoc
  \usepackage[nocompress]{cite}
\else
  \usepackage{cite}
\fi
\ifCLASSINFOpdf
   \usepackage[pdftex]{graphicx}
\else
   \usepackage[dvips]{graphicx}
\fi
\usepackage{amsmath}
\usepackage{algorithmic}

\usepackage{array}

\ifCLASSOPTIONcompsoc
  \usepackage[caption=false,font=footnotesize,labelfont=sf,textfont=sf]{subfig}
\else
  \usepackage[caption=false,font=footnotesize]{subfig}
\fi
\usepackage{bm}
\usepackage{booktabs}
\usepackage{color}
\usepackage{stfloats} 
\usepackage{amsthm,amssymb,amsfonts}
\usepackage{algorithm}
\usepackage{algorithmic}
\begin{document}
%
\title{Collective Decision for Open Set Recognition}
%
%
%
%

\author{Chuanxing~Geng~and~Songcan~Chen
\IEEEcompsocitemizethanks{\IEEEcompsocthanksitem The authors are with College of Computer Science and Technology, Nanjing University of Aeronautics and Astronautics, MIIT Key Laboratory of Pattern Analysis and Machine Intelligence, Nanjing, 211106, China. E-mail: \{gengchuanxing, s.chen\}@nuaa.edu.cn.}
\thanks{}}

%
%

\markboth{Journal of \LaTeX\ Class Files,~Vol.~14, No.~8, August~2015}%
{Shell \MakeLowercase{\textit{et al.}}: Bare Demo of IEEEtran.cls for Computer Society Journals}
%



\IEEEtitleabstractindextext{%
\begin{abstract}
  In open set recognition (OSR), almost all existing methods are designed specially for recognizing individual instances, even these instances are collectively coming in batch. Recognizers in decision either reject or categorize them to some known class using empirically-set threshold. Thus the decision threshold plays a key role. However, the selection for it usually depends on the knowledge of known classes, inevitably incurring risks due to lacking available information from unknown classes. On the other hand, a more realistic OSR system should NOT just rest on a reject decision but should go further, especially for discovering the hidden unknown classes among the reject instances, whereas existing OSR methods do not pay special attention. In this paper, we introduce a novel collective/batch decision strategy with an aim to extend existing OSR for new class discovery while considering correlations among the testing instances. Specifically, a collective decision-based OSR framework (CD-OSR) is proposed by slightly modifying the Hierarchical Dirichlet process (HDP). Thanks to HDP, our CD-OSR does not need to define the decision threshold and can implement the open set recognition and new class discovery simultaneously. Finally, extensive experiments on benchmark datasets indicate the validity of CD-OSR.
\end{abstract}

}

\maketitle

\IEEEdisplaynontitleabstractindextext

%
\IEEEpeerreviewmaketitle

\IEEEraisesectionheading{\section{Introduction}\label{sec:introduction}}

%
%
%
%
\IEEEPARstart{I}{n} real-world recognition/classification tasks, limited by various objective factors, it is usually difficult to collect training instances exhaustively of all classes when training a classifier. A more realistic scenario is open set recognition (OSR) \cite{Scheirer2013Toward}, where incomplete knowledge of the world exists at training time, and unknown classes can be submitted to an algorithm during testing. This requires the classifiers to not only accurately classify the seen known classes but also effectively deal with the unknown ones.

The main challenge for OSR is that the traditional classifiers usually trained under closed set scenario divide over-occupied space for known classes, thus resulting in misclassifying the instances of unknown classes unseen in training as the known classes. To meet this challenge, related studies have been conducted under a number of frameworks, assumptions and names \cite{Phillips2005Evaluation,Li2005Open,wu2007novel,wang2009A,Heflin2012Detecting,Pritsos2013Open}. For example, Phillips et al. \cite{Phillips2005Evaluation} proposed a typical framework for open set identity recognition in a study on evaluation methods for face recognition, while Li and Wechsler \cite{Li2005Open} again viewed open set face recognition from an evaluation perspective and proposed the Open Set TCM-kNN algorithm. It is Scheirer et al. \cite{Scheirer2013Toward} that first formalized the open set recognition problem and proposed a preliminary solution---1-vs-Set machine, which incorporates an open space risk term in modeling to account for the space beyond the reasonable support of known classes. Although 1-vs-Set machine decreases the region of known class for each binary support vector machine (SVM), the space occupied by each known class remains unbounded. Therefore, the open space risk still exists. As shown in Fig. 1, the 1-vs-Set machine will make misclassifications if the instances of unknown classes ?2,?3 appear in testing. To overcome this problem, researchers have further made many efforts.

\begin{figure}[!t]
\centering
\includegraphics[width=8.7cm]{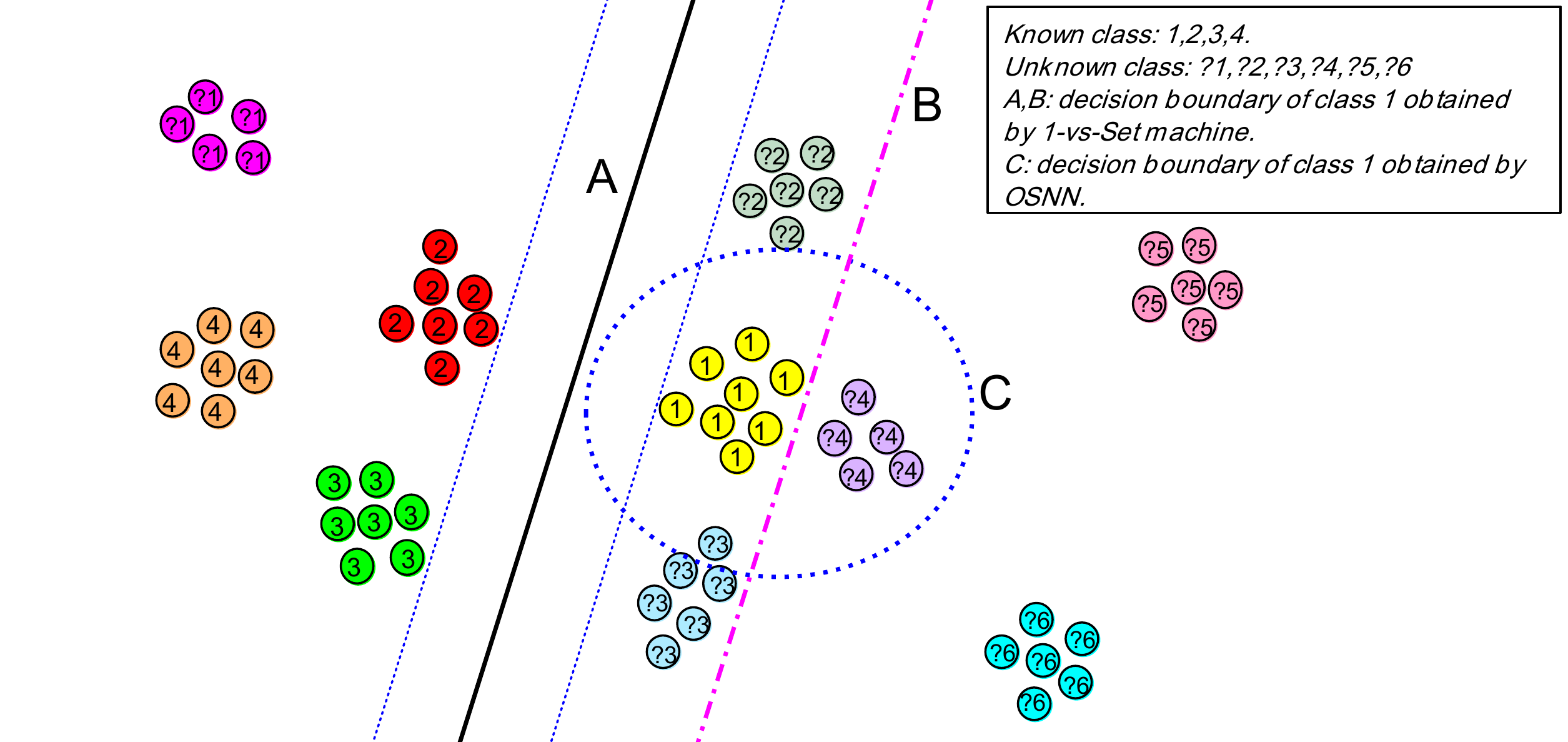}
\caption{Only known classes 1-4 are available in training, while unknown classes ?1-?6 appear during testing. 'A' and 'B' are the decision boundaries of class 1 obtained by 1-vs-Set machine, while 'C' is the decision boundary of class 1 obtained by OSNN.}
\end{figure}

Scheirer et al. \cite{Scheirer2014Probability} incorporated non-linear kernels into a solution that further limited the open space risk by positively labeling only set with finite measure. Specifically, they proposed a novel Weibull-calibrated SVM (W-SVM), which combines the statistical extreme value theory (EVT) for score calibration with one-class and binary SVMs. Intuitively, we can reject the large set of unknown classes (even under an assumption of incomplete class knowledge) if the positive data for any known classes is accurately modeled without overfitting. Based on this intuition, Jain et al. \cite{jain2014multi} invoked EVT to model the positive training data at the decision boundary and proposed the $P_I$-SVM algorithm. Note that both W-SVM and $P_I$-SVM adopt the threshold-based classification scheme, thus the thresholds play a key role. However, the thresholds in those models are usually assumed to be equal for all known classes, which is not reasonable since the distributions of known classes in feature space are unknown. On the other hand, the authors in \cite{Scheirer2014Probability,jain2014multi} recommended setting these thresholds according to the problem openness. But unfortunately the openness of the corresponding problem usually is also unknown. To overcome these deficiencies, Scherreik et al. \cite{Scherreik2016Open} proposed the probabilistic open set SVM classifier (POS-SVM), where the unique reject threshold for each known class is empirically determined.

Recently, J$\acute{\text{u}}$nior et al. \cite{junior2017nearest} extended the Nearest-Neighbor classifier to OSR scenario and proposed the OSNN classifier. Zhang et al. \cite{zhang2017sparse} proposed the SROSR algorithm based on sparse representation, where they modeled the tails of the matched and sum of non-matched reconstruction error distributions using EVT. Taking distributional information into account when learning recognition functions, Rudd et al. \cite{rudd2018extreme} formulated a theoretically sound classifier---the Extreme Value Machine (EVM) which is further developed in \cite{Vignotto2018Extreme}. Besides, researchers also explored open set recognition based on deep neural networks \cite{Bendale2015Towards1,Rozsa2017Adversarial,Hassen2018Learning,Shu2017DOC,Cardoso2015A,Cardoso2017Weightless,Shu2018Unseen,ge2017generative,Xu2018Learning}.

In summary, all current existing OSR algorithms are designed specially for recognizing individual instances, even these instances are all arriving collectively in batch like image-set recognition \cite{Wang2017Joint}. Only one decision that so-designed recognizer can make is to either reject or categorize them to some known class instance by instance using some empirically-set threshold. Thus the decision threshold plays a key role. However, the selection for it is usually based on the knowledge of known classes. This inevitably incurs risks due to no available information from unknown classes. As shown in Fig. 1, the decision boundary\footnote{The decision boundary of a class defines the region in which a possible testing sample will be classified as belonging to that class.} 'C' obtained by OSNN for known class 1 can reject the large set of unknown classes, whereas it still makes a misclassification when unknown class ?4 appears in testing.

On the other hand, a more realistic or desired OSR system should NOT just rest on a reject decision but should go further, especially for discovering the hidden unknown classes among the reject instances. Unfortunately, existing OSR methods do not pay special attention to this point. Although Bendale and Boult \cite{Bendale2015Towards} introduced the open world recognition framework which can collect and label (e.g. by humans) the reject instances to further use for updating the OSR model, it is actually a post-event strategy needing human intervention. Meanwhile, the authors in \cite{Shu2018Unseen} transferred the knowledge of the similarity and difference in known classes for new class rediscovery among the already-rejected instances. Obviously, this is still a post-event approach. In fact, such a two-step manner \cite{Bendale2015Towards,Shu2018Unseen} easily incurs a suboptimal solution. Therefore, it is necessary to specifically design a model so that both open set recognition and new class discovery can be proceeded simultaneously.

Towards this goal, in this paper, we attempt to introduce a novel collective/batch decision strategy for OSR. As a Bayesian nonparametric modeling method, hierarchical Dirichlet process (HDP) does not overly depend on the training data and can achieve adaptive change as the data changes \cite{gershman2012tutorial}. In particular, for the new coming instance, HDP has the ability to assign this instance an existing subclass or a new subclass drawn from the base distribution (the details c.f. subsection 3.1). Such a property makes HDP automatically reserve space for unknown classes in testing, naturally leading to a new class discovery function. Therefore, with slight modification to HDP, we propose a collective decision-based OSR framework (CD-OSR)
as an initial solution towards the open set recognition of collective decision. Note that the HDP can also be replaced by other Bayesian nonparametric techniques \cite{gershman2012tutorial} like the hierarchical beta process \cite{thibaux2007hierarchical}, but beyond our focus here. Thanks to the properties of HDP, CD-OSR does not need to define the decision threshold and can implement the open set recognition and new class discovery simultaneously. Additionally, treating the testing instances in batch makes CD-OSR take into account correlations among the instances obviously ignored by existing methods. Note that CD-OSR actually can handle both batch and individual instances. Specifically, the contributions and details of our CD-OSR can be highlighted as follows:
%
%
%
%
%
%
\begin{enumerate}[\IEEEsetlabelwidth{8)}]
\item A novel collective/batch decision strategy is first introduced for open set recognition, which can address the instances \emph{in batch}, even \emph{individual} instances. Specifically, a collective decision-based OSR framework (CD-OSR) is proposed, which can address both the open set recognition and new class discovery simultaneously.
\item CD-OSR does not need to define the decision threshold and can automatically reserve space for unknown classes in testing, naturally leading to the new class discovery function.
\item Treating the testing instances in batch makes CD-OSR consider correlations among the instances obviously ignored by the other existing OSR motheds.
\item A thorough empirical evaluation of CD-OSR is reported, showing the significant improvement in classification performance and the new class discovery function.
\end{enumerate}

The remainder of this paper is organized as follows. Section 2 gives the related work in open set recognition. Section 3 introduces a novel collective decision strategy for OSR, where a collective decision-based OSR framework is specifically given. Experimental evaluation is reported in Section 4. Finally, Section 5 gives a conclusion.

\section{Related Work}
With the formalization of OSR developed in \cite{Scheirer2013Toward}, the openness of a particular problem or data universe is defined by considering the number of training, target, and testing classes:
\begin{equation}
\text{openness} = 1 - \sqrt{\frac{2\times|\text{training classes}|}{|\text{testing classes}|+|\text{target classes}|}}.
\end{equation}
Larger openness corresponds to more open problems, while the problem is completely closed when $\text{openness}=0.$ Furthermore, the OSR problem can be defined as follows: given a set of training data $U$, an open space risk $R_{\mathcal{O}}$, and an empirical risk $R_{\varepsilon}$, the goal of OSR is to find a measurable recognition function $f\in \mathcal{H}$ defined by minimizing the following open set risk
\begin{equation}
f = \arg\min_{\hat{f}\in \mathcal{H}}\{R_{\mathcal{O}}(\hat{f}) + \lambda_rR_{\varepsilon}(\hat{f}(U))\},
\end{equation}
where $\lambda_r$ is a regularization constant. Thanks to the guidance of this definition, a large number of OSR algorithms have been proposed. Next, we will briefly review the relevant representative approaches.

\subsection{The Existing OSR Methods}

\subsubsection{The 1-vs-Set Machine}
Using the definition of OSR, an SVM-based OSR method called the 1-vs-Set machine \cite{Scheirer2013Toward} is proposed. In 1-vs-Set, the open space risk $R_{\mathcal{O}}$ is considered to be the ratio of the Lebesgue measure of positively labeled open space compared to the overall measure of positively labeled space. Concretely, a hyperplane 'B' (shown in Fig. 1) parallelling the separating hyperplane 'A' obtained by the SVM is derived in score space, leading to a slab in feature space. Thus, the open space risk for a linear kernel slab model is defined as follows:
\begin{equation}
R_{\mathcal{O}} = \frac{\zeta_B - \zeta_A}{\zeta^{+}} + \frac{\zeta^{+}}{\zeta_B - \zeta_A} + p_A\omega_A + p_B\omega_B,
\end{equation}
where $\zeta_A$ and $\zeta_B$ denote the marginal distances of the corresponding hyperplanes, and $\zeta^{+}$ is the separation needed to account for all positive data. Additionally, user-specified parameters $p_A$ and $p_B$ are given to weight the importance between the margin spaces $\omega_A$ and $\omega_B$.

After training the 1-vs-Set machine, a testing instance that appears between the two hyperplanes would be labeled with the appropriate class. Otherwise, it is considered as non-target class or rejected, depending on which side of the slab it resides. As discussed in Section 1, the 1-vs-Set machine reduces the open space risk to some extent. However, it still occupies the infinite space, meaning the open space risk still exists.

\subsubsection{The W-SVM Model}
To further reduce the open space risk, Scheirer et al. \cite{Scheirer2014Probability} incorporated non-linear kernels into a solution that further limited open space risk by positively labeling only sets with finite measure. They formulated a compact abating probability (CAP) model, where probability of class membership abates as points move from known data to open space. Specifically, a Weibull-calibrated SVM (W-SVM) model was proposed, which combined the EVT for score calibration with two separated SVMs. The first is a one-class SVM CAP model used as a conditioner: if the posterior estimate $P_O(y|x)$ of an input instance $x$ predicted by one-class SVM is less than a threshold $\rho_\tau$, the instance will be rejected outright. Otherwise, it will be passed to the second SVM. The second one is a binary SVM CAP model via a fitted Weibull cumulative distribution function, yielding the posterior estimate $P_{\eta}(y|x)$ for the corresponding positive class. Furthermore, it also obtains the posterior estimate $P_{\psi}(y|x)$ for the corresponding negative classes by a reverse Weibull fitting. Defined an indicator variable: $\iota_y=1$ if $P_O(y|x)>\rho_\tau$ and $\iota_y=0$ otherwise, then the W-SVM model for OSR is defined as follows
\begin{eqnarray}
\begin{split}
y^* &= \ \arg\max\limits_{y\in\mathcal{Y}} P_\eta(y|x)\times P_{\psi}(y|x)\times \iota_y  \\
    &\text{subject to}\ \  P_\eta(y^*|x)\times P_{\psi}(y^*|x) \geq \rho_R,
\end{split}
\end{eqnarray}
where $\mathcal{Y}$ denotes all the known classes, and $\rho_R$ is the threshold of the second SVM CAP model.

Additionally, the thresholds $\rho_\tau$ and $\rho_R$ are set empirically, e.g., $\rho_\tau$ is fixed to 0.001 as specified by the authors, while $\rho_R$ is recommend to set depending on the openness of the specific problem by
\begin{equation}
\rho_R = 0.5\times \text{openness}.
\end{equation}
The W-SVM effectively limits the open space risk by the threshold-based classification schemes. However, such a threshold setting, especially for $\rho_R$, is difficult since we usually have no prior knowledge about unknown classes.

\subsubsection{The OSNN Model}
Adapting the traditional closed-set Nearest Neighbor classifier to the OSR scenario, J$\acute{\text{u}}$nior et al. \cite{junior2017nearest} proposed the OSNN classifier. Let $\vartheta(x)\in \mathcal{L}=\{\ell_1,\ell_2,...,\ell_n\}$ represent the class of the corresponding instance $x$ and $\mathcal{L}$ be the set of training labels (known classes). The OSNN first finds the nearest neighbor $u_1$ and $u_2$ of testing instance $s$, where $\vartheta(u_1)\neq\vartheta(u_2)$. Then one can calculate the ratio $\upsilon=d(s,u_1)/d(s,u_2)$, in which $d(x,x')$ is the Euclidean distance between instances $x$ and $x'$ in the feature space. If $\upsilon$ is less than or equal to the predefined threshold $\sigma$ ($0<\sigma<1$), $s$ is classified as the same label of $u_1$. Otherwise, it is considered as unknown, i.e.,
\begin{equation}
   \vartheta(s) = \left\{
   \begin{aligned}
   &\vartheta(u_1) \ \ \ \ \text{if}\ \  \upsilon\leq\sigma  \\
   &\text{"unknown"} \ \ \ \text{if}\ \  \upsilon>\sigma\\
   \end{aligned}
   \right.
\end{equation}

Note that applying a threshold on the ratio of similarity scores seems better than on the similarity scores themselves as reported in \cite{junior2017nearest}. However, the selection of such a threshold is still an empirical setting, inevitably incurring risks due to lacking available information from unknown classes. As described in Section 1, the OSNN will make a misclassification when unknown class ?4 appears in testing. In addition, just selecting two reference instances from different classes for comparison makes the OSNN model vulnerable for outliers.


\subsection{Unseen Class Discovery in existing OSR}
In fact, there are also some researchers paying attention to the unknown class discovery in OSR. To further extend open set recognition, the authors in \cite{Bendale2015Towards} formalized the open world recognition problem: a recognition system should perform four tasks including detecting unknown classes (open set recognition), choosing which instances to label for addition to the model, labelling those instances, and then updating the classifier. Ideally, all of these tasks should be automated. But in \cite{Bendale2015Towards}, the authors just presumed supervised learning with labels obtained by human labeling. In addition, the unknown class discovery in \cite{Bendale2015Towards} is a post-event strategy, meaning that people must first obtain the rejected instances before proceeding to the new class's discovery. Actually, such a two-step manner easily incurs a suboptimal solution.

Besides, Shu et al. \cite{Shu2018Unseen} mainly focused on discovering the unknown classes hiding among the reject instances by transferring the knowledge of the similarity and difference in known classes. Correspondingly, a joint open classification framework was proposed with four components: an Open Classification Network (OCN) used for open set recognition, a Pairwise Classification Network (PCN) for classifying whether two input instances are from the same class or not, an auto-encoder for learning representation from unlabeled instances, and a hierarchical clustering model for clustering the reject instances. Similar to \cite{Bendale2015Towards}, this approach adopts a two-step manner as well. Furthermore, the use of knowledge in known classes is risky when the transferred knowledge in known and unknown classes differs.

\section{Collective Decision for Open Set Recognition}
As discussed previously, the current existing OSR methods are designed specially for recognizing individual instances, even these instances are all arriving collectively in batch. Hence recognizers in decision either reject or categorize them to some known class instance by instance using empirically-set threshold, where such a decision threshold plays a key role. However, its selection is usually based on the knowledge of known classes, inevitably incurring risks due to no available information from unknown classes. On the other hand, a more realistic OSR system should NOT just rest on a reject decision but should go further, especially for discovering the hidden unknown classes in the reject instances. Regrettably, existing OSR methods do not pay much attention. Although \cite{Bendale2015Towards,Shu2018Unseen} have made some efforts, they are just a two-step strategy.

To overcome these limitations mentioned above, we introduce a novel collective decision strategy for OSR problem with an aim to extend existing open set recognition for new class discovery. Specifically, a collective decision-based OSR framework (CD-OSR) is proposed by slightly modifying HDP. Thanks to the properties of HDP, Our CD-OSR does not need to define the decision threshold and can automatically reserve space for unknown classes in testing, naturally leading to the new class discovery function. Interestingly, this also makes it able to handle OSR and new class discovery at the same time. Moreover, treating the testing instances in batch makes CD-OSR consider correlations among the instances obviously ignored by existing methods. Note that CD-OSR actually can handle both batch and individual instances.

Next, we first give a brief review of HDP \cite{Teh2006Hierarchical} widely used for co-clustering multiple groups of data by sharing mixture components among the groups. In HDP, the commonly used terms are 'group' or 'component'. However, we here adapt HDP with slight modification to the OSR problem. Under the OSR scenario, 'class' actually corresponds to 'group', while 'subclass' corresponds to 'component'. Therefore, in order to avoid confusion, we unify these terms ('class' $\leftrightarrow$ 'group', 'subclass' $\leftrightarrow$ 'component') throughout this paper.

%

\subsection{Hierarchical Dirichlet Process}
The Dirichlet process (DP) \cite{teh2011dirichlet,Gershman2012A} considered as a distribution over distributions is a stochastic process, which is mainly used in clustering and density estimation problems as a nonparametric prior defined over the number of mixture components. As a hierarchical extension to DP, the Hierarchical Dirichlet Process \cite{Teh2006Hierarchical} is proposed, modeling each group of data in the form of a Dirichlet process mixture model (DPM). Under this hierarchical structure, an elegant way of sharing parameters is provided, allowing the DPM models across different groups to be connected together through a higher level DP.

Let $x_{ji}\in R^d,\ i=\{1,...,n_j\},\ j=\{1,...,J\}$ denote the instance $i$ in the group $j$ where $n_j$ denotes the number of instances in group $j$, $J$ is the total number of groups, and $\theta_{ji}$ represents the parameters of the mixture component associated with $x_{ji}$. Then the HDP framework is completed as follows:
\begin{equation}
  \begin{split}
   &G_0|\gamma,H \sim \text{DP}(\gamma,H)    \\
   &G_j|\alpha_0,G_0 \sim \text{DP}(\alpha_0,G_0)\ \text{for each}\ j\\
   &\theta_{ji}|G_j \sim G_j\ \ \text{for each $j$, $i$}\\
   &x_{ji}|\theta_{ji} \sim F(\theta_{ji})\ \ \text{for each $j$, $i$},
  \end{split}
\end{equation}
where $G_0$ as a global distribution is distributed as a Dirichlet process with concentration parameter $\gamma$ and base distribution $H$. $G_j$ for each group is distributed according to the DP with concentration parameter $\alpha_0$ and base distribution $G_0$. Moreover, as $\alpha_0$ increases, the number of components (or clusters) used to represent each group data increases. Note that although increasing $\gamma$ can add the clusters used to represent the data of all groups, the degree to which these clusters are shared between groups will decrease at the same time \cite{Canini2010Modeling}. Besides, $x_{ji}$ can also been viewed as a draw from a distribution $F(\theta_{ji})$.

An intuitive understanding of the generative process defined by a HDP model can be through an analogy to the Chinese Restaurant Franchise (CRF). CRF actually extends upon the Chinese restaurant process (CRP), allowing multiple restaurants to share a set of dishes. In the CRF metaphor, customer $i$ in restaurant $j$ is associated with $\theta_{ji}$ and sits at table $t_{ji}$. The table $t$ is associated with one of the $K$ random draws from $H$, i.e., $\psi_{jt}\in\{\phi_1,...,\phi_K\}$ denoting the global menu of dishes. Moreover, a dish from the global menu served at table $t$ in restaurant $j$ is represented by the indicator variable $k_{jt}$. In addition, the concentration parameter $\gamma$ controls the prior probability of serving a new dish at a new table \cite{akova2012self}.

In this framework, the restaurants correspond to groups, the tables in each restaurant correspond to the mixture components in the DP mixture model, and the dishes in the global menu correspond to the unique set of parameters shared among the restaurants.

The conditional distributions for $\theta_{ji}$ and $\psi_{jt}$ can be obtained by integrating out $G_j$ and $G_0$, respectively.
\begin{equation}
\theta_{ji}|\theta_{j1},...,\theta_{j,i-1},\alpha_0,G_0 \sim \sum_{t=1}^{m_{j\cdot}}\frac{n_{jt\cdot}}{i-1+\alpha_0}\delta_{\psi_{jt}}+\frac{\alpha_0}{i-1+\alpha_0}G_0,
\end{equation}
where $m_{j\cdot}$ represents the number of tables in restaurant $j$, $n_{jt\cdot}$ is the number of customers in restaurant $j$ at table $t$, and $\delta_{\psi_{jt}}$ denotes the Dirac measure\footnote{The definition of Dirac measure is given in Supplementary Material.} at $\psi_{jt}$. According to (8), the conditional $\theta_{ji}$ is assigned to one of the existing $\psi_{jt}$ with probability $\frac{n_{jt\cdot}}{i-1+\alpha_0}$ or $\psi_{j,m_{j\cdot+1}}$, i.e., a new table drawn from $G_0$ with probability $\frac{\alpha_0}{i-1+\alpha_0}$. Marginal counts are represented with dots. Similarly,
\begin{eqnarray}
\psi_{jt}|\psi_{11},\psi_{12},...,\psi_{21},...,\psi_{jt-1},\gamma,H \sim \\ \nonumber
\sum_{k=1}^{K}\frac{m_{\cdot k}}{m_{\cdot\cdot}+\gamma}\delta_{\phi_k}+\frac{\gamma}{m_{\cdot\cdot}+\gamma}H,
\end{eqnarray}
where $m_{\cdot k}$ represents the number of tables across all restaurants serving dish $\phi_k$, $m_{\cdot\cdot}$ denotes the total number of tables occupied by all restaurants, and $\delta_{\phi_k}$ denotes the Dirac measure at $\phi_k$. According to (9), the conditional distribution $\psi_{jt}$ inherits one of the existing $\phi_k$ with probability $\frac{m_{\cdot k}}{m_{\cdot\cdot}+\gamma}$ or $\phi_{K+1}$, i.e., a new dish drawn from $H$ with probability $\frac{\gamma}{m_{\cdot\cdot}+\gamma}$.

The inference of CRF can be performed by using a Gibbs sampling scheme \cite{Teh2006Hierarchical,ishwaran2001gibbs}. In order to simplify the derivation without losing the general applicability, the base distribution $H$ here is conjugate to the data distribution $F$. Furthermore, for ease of understanding, the notations used here are the same as those in \cite{Teh2006Hierarchical}. Let $z_{ji} = k_{jt_{ji}}$ represent $x_{ji}$'s index variable. Since $t_{ji}$ and $k_{jt}$ are the respective index variables of $\theta_{ji}$ and $\psi_{jt}$ ($\theta_{ji}=\psi_{jt_{ji}}$, $\psi_{jt} = \phi_{k_{jt}}$, the prior of $\phi_{k_{jt}}$ is $H$), we actually sample $t_{ji}$ and $k_{jt}$ instead of dealing with $\theta_{ji}$'s and $\psi_{jt}$'s directly. Let $\mathbf{x}=(x_{ji}:\text{all $j,i$})$, $\mathbf{x}_{jt}=(x_{ji}: \text{all $i$ with $t_{ji} = t$})$, $\mathbf{t}=(t_{ji}:\text{all $j,i$})$, $\mathbf{k}=(k_{jt}:\text{all $j,t$})$, $\mathbf{z}=(z_{ji}:\text{all $j,i$})$, $\mathbf{m}=(m_{jk}:\text{all $j,k$})$ and $\bm{\phi}=(\phi_1,...,\phi_K)$ be described like those in \cite{Teh2006Hierarchical}. Let $\mathbf{t}^{-ji}$, $\mathbf{k}^{-jt}$ or $n_{jt}^{-ji}$, $m_{\cdot k}^{-jt}$ respectively denote the corresponding  superscripts removed from the sets or from the calculation of the counts. Let $f(\cdot|\theta)$ and $h(\cdot)$ respectively denote the densities $F(\theta)$ and $H$. Integrating out the mixture component parameters $\bm{\phi}$, we can obtain $x_{ji}$'s conditional density under mixture component $k$ given all data items except $x_{ji}$ as
\begin{equation*}
f_k^{-x_{ji}}(x_{ji}) = \frac{\int f(x_{ji}|\phi_k)\prod_{j'i'\neq ji,z_{j'i'}=k}f(x_{j'i'}|\phi_k)h(\phi_k)d\phi_k}{\int\prod_{j'i'\neq ji,z_{j'i'}=k}f(x_{j'i'}|\phi_k)h(\phi_k)d\phi_k}.
\end{equation*}
Similarly, the conditional density of $x_{jt}$, i.e., $f_k^{-x_{jt}}(x_{jt})$, can be obtained given all data items associated with mixture component $k$ leaving out $x_{jt}$. Note that the conditional distributions here have omitted the conditions such as \emph{concentration} parameters, e.g., $f_k^{-x_{ji}}(x_{ji})=f_k^{-x_{ji}}(x_{ji}|\bm{x}^{-{ji}},\alpha_0,\gamma)$. Then the conditional distribution of $t_{ji}$ is
\begin{eqnarray}
   \begin{split}
   &p(t_{ji}=t|\mathbf{t}^{-ji},\mathbf{k}) \propto\\
   &\left\{
   \begin{aligned}
   &n_{jt\cdot}^{-ji}f_{k_{jt}}^{-x_{ji}}(x_{ji}),&&\text{if $t$ previously used,} \\
   &\alpha_0p(x_{ji}|\mathbf{t}^{-ji},t_{ji}=t^{\text{new}},\mathbf{k}),&&\text{if $t=t^{\text{new}}$},
   \end{aligned}
   \right.
   \end{split}
\end{eqnarray}
where
\begin{eqnarray*}
&&p(x_{ji}|\mathbf{t}^{-ji},t_{ji}=t^{\text{new}},\mathbf{k})=\\
&&\sum_{k=1}^K\frac{m_{\cdot k}}{m_{\cdot\cdot}+\gamma}f_k^{-x_{ji}}(x_{ji})+\frac{\gamma}{m_{\cdot\cdot}+\gamma}f_{k^{\text{new}}}^{-x_{ji}}(x_{ji}),
\end{eqnarray*}
where $f_{k^{\text{new}}}^{-x_{ji}}(x_{ji})=\int f(x_{ji}|\phi)h(\phi)d\phi$ is simply the prior density of $x_{ji}$. Furthermore, the conditional distribution of $k_{jt}$ is
\begin{eqnarray}
   \begin{split}
   &p(k_{ji}=k|\mathbf{k}^{-jt},\mathbf{t}) \propto \\
   &\left\{
   \begin{aligned}
   &m_{\cdot k}^{-jt}f_k^{-x_{jt}}(x_{jt}),&&\text{if $k$ previously used,}  \\
   &\gamma f_{k^{\text{new}}}^{-x_{jt}}(x_{jt}),&&\text{if $k=k^{\text{new}}$}.
   \end{aligned}
   \right.
   \end{split}
\end{eqnarray}

\subsection{CD-based Open Set Recognition}
Since the properties of hierarchical Dirichlet process (HDP) described above fits our problem, we here adapt HDP with slight modification to
addressing the OSR problem. Thus a collective decision-based OSR framework (CD-OSR) is proposed as an initial solution towards open set recognition of collective decision. Concretely, the CD-OSR works as follows.

\begin{figure}[!t]
\centering
\includegraphics[width=8cm]{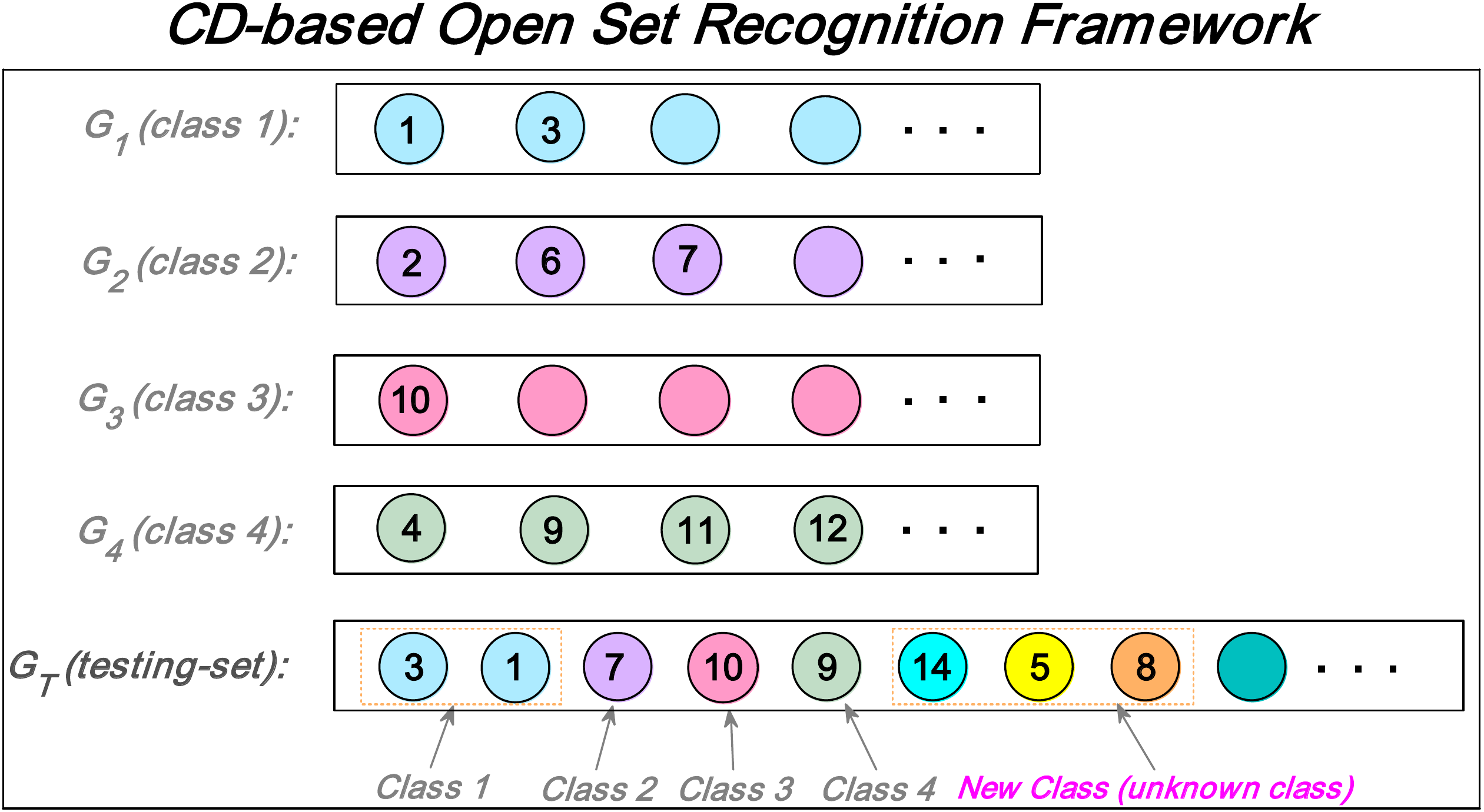}
\caption{Each known class (here is class 1-4), as a group in CD-OSR, is modeled by a Dirichlet Process while the testing set (including unknown classes or not) as a whole is treated in the same way. Then all of the groups are co-clustered under the CD-OSR framework. For a testing instance, it would be labeled as the appropriate known class or unknown class, depending on whether the subclass this sample is assigned associates with the corresponding known class or not. The number in the circle indicates the corresponding subclass.}
\end{figure}
\textbf{(1) Training Phase}: In our CD-OSR framework, we first divide the training set into a fitting set $\mathcal{F}$ and a validation set $\mathcal{V}$ (the details are given in subsection 4.1.1). Next, we model each known class data in $\mathcal{F}$ as a group of HDP using a Gaussian mixture model (GMM) with an unknown number of components. Simultaneously, the whole validation set $\mathcal{V}$ as one batch is treated in the same way. Then all of the groups are co-clustered under the HDP framework. Unlike HDP, we also \emph{append a parameter $\varrho$ denoting the proportion of the corresponding subclass in its class in CD-OSR}. If the $\varrho$ is below some constant $\varepsilon$ after co-clustering, the corresponding subclass intuitively should be omitted for avoiding the overfitting. Note that the role of $\varepsilon$ should not be confused with the thresholds of the existing OSR methods which are used to determine the boundary between known and unknown classes. Then
\emph{repeat this process several times to preform a grid search operation on the corresponding candidate parameter set, and obtain the appropriate initialization parameter values for CD-OSR.} Note that these parameters actually do not overly depend on training data due to the properties of HDP.



\textbf{(2) Testing Phase}: Fixing the appropriate initialization parameters achieved in training, we will obtain our CD-OSR recognition framework. Similar to the training phase, we model each known class data in \textbf{training set} as a group of the CD-OSR, while the whole testing set as one collective/batch\footnote{This kind of operation is completely for convenience. In fact, the size of batch does not significantly influence the classification performance, and the subsection 4.2.2 reports this result in detail.} is treated in the same way. Then all of the groups are co-clustered  under CD-OSR. After the co-clustering process, each class will obtain one or many subclasses. Note that the key to classification is whether the subclass assigned to the testing instance is included in the corresponding known class or not: if yes, the instance is labeled as the appropriate known class; otherwise it will be recognized as unknown class, as shown in Fig. 2. Furthermore, Algorithm 1 also shows the workflow of CD-OSR.
\begin{table}[!h]
\tabcolsep 0pt
\vspace*{-5pt}
\begin{center}
\def\temptablewidth{0.49\textwidth}
{\rule{\temptablewidth}{1.5pt}}
\begin{tabular*}{\temptablewidth}{@{\extracolsep{\fill}}l}
\bf Algorithm 1 CD-OSR  \\   \hline
{\bf Initializing} \\
\vspace*{0.1cm}
\hspace{0.2cm} 1. Let $\bm{X}_{tr}$ and $\bm{X}_{ts}$ respectively denote the Training Set and Testing \\
\hspace{0.5cm}    Set obtained by the experimental protocol in subsection 4.1.1.              \\
\vspace*{0.1cm}
\hspace{0.2cm} 2. According to the class labels, Divide the Training Set \\
\hspace{0.5cm}   $\bm{X}_{tr}=[\bm{X}_{tr1};\bm{X}_{tr2};...]$ sequentially class by class. \\
\vspace*{0.1cm}
\hspace{0.2cm} 3. Initializing parameters: $\mu_0,\Sigma_0,\beta,\nu,\varsigma,\alpha_0,\gamma,H,\varepsilon,T$, InS described \\
\hspace{0.5cm}    in the paper. \\
\vspace{0.1cm}
{\bf Co-clustering under the HDP framework} \\
\vspace*{0.1cm}
\hspace{0.2cm} Results = HDP($[\bm{X}_{tr};\bm{X}_{ts}]$, $\alpha_0$, $\gamma$, $H$, InS, $T$)  \\
\vspace{0.1cm}
{\bf Determining the subclasses} \\
\vspace*{0.1cm}
\hspace{0.2cm} After co-clustering, let $\varrho$ denote the proportion of the corresponding   \\
\hspace{0.2cm} subclass in its class: the subclass will be removed, if $\varrho<\varepsilon$.\\
\vspace{0.1cm}
{\bf Predicting} \\
\vspace*{0.1cm}
\hspace{0.2cm} For a new coming instance, it will be labeled according to the \\
\hspace{0.2cm} following rule:\\
\hspace{0.8cm} 1) a corresponding known class, if the subclass assigned to it \\
\hspace{1.1cm}    belongs to some known class.\\
\hspace{0.8cm} 2) an unknown class, if the subclass assigned to it comes from\\
\hspace{1.1cm}    a new draw from $H$.
\end{tabular*}
{\rule{\temptablewidth}{1pt}}
The HDP here is the software package from  \cite{Teh2006Hierarchical} used to implement the Hierarchical Dirichlet Process
\end{center}
\end{table}


Note that the testing phase is nothing but a co-clustering process, which seems to have the flavor of lazy learning to some extent. Furthermore, the collective/batch operation for the testing set makes our CD-OSR can address the instances \emph{in batch}, even \emph{individual} instances. Unlike existing OSR methods which infer unknown classes depending on the empirically-set decision threshold, our CD-OSR does not need to define such a threshold and can provide explicit modeling for the unknown classes appearing in testing. This naturally endows it new class discovery capability which will be detailed in subsection 4.3. Such a capability intuitively makes our CD-OSR have zero open space risk under ideal conditions where all classes including known and unknown classes are mutually exclusive. Moreover, under the CD-OSR framework, each new/unknown class will inherently have only one subclass as we have no available knowledge from unknown classes. Note that unlike the two-step manner in \cite{Bendale2015Towards,Shu2018Unseen}, CD-OSR actually is a jointly solving manner due to the co-clustering of HDP. In addition, the collective operation in CD-OSR also makes our framework consider the correlations among the testing instances obviously ignored by other existing OSR methods.

Besides, the key to accurate prediction of our CD-OSR is the sharing of subclasses between the testing set's group and the groups of the training set. However, the known classes of the training data may also share the same subclasses between themselves, resulting in an unidentifiable problem. Therefore, we usually set a lager $\gamma$ to decrease the degree to which the subclasses are shared between those classes. Intuitively, if all classes (including known and unknown classes) are mutually exclusive, the subclasses associated with the different classes would be different, making the input instances identifiable. Furthermore, we state the following proposition.
\newtheorem{proposition}{Proposition}
\begin{proposition}
Assume the set of potential classes, i.e., known and unknown, are mutually exclusive. Let $m_{\cdot k}$, $m_{\cdot\cdot}$, $\gamma$, $\phi_k$, and $H$ be described as above, and $L$ denote the number of subclasses associated with the corresponding known classes. Then our CD-OSR framework can model the subclasses of the new coming instances associated with the corresponding known classes with probability $\sum_{k=1}^L\frac{m_{\cdot k}}{m_{\cdot\cdot}+\gamma}\delta_{\phi_k}$ or unknown classes with probability $\frac{\gamma}{m_{\cdot\cdot}+\gamma}$, whilst it would have zero open space risk.
\end{proposition}
\begin{proof}
This proposition can be obviously obtained from the generative process of HDP.
\end{proof}

\subsubsection{Computational Complexity Analysis}
Using the Gaussian-Wishart distribution (the details c.f. subsection 4.1.2) as the base distribution $H$, CD-OSR adopts the Gibbs sampling scheme developed by \cite{Teh2006Hierarchical} to implement the inference process. In this scheme, the only computations needed are \emph{marginal likelihood} (ML) computations, and \emph{posterior} (PO) computations for the parameters in the Gaussian-Wishart distribution\footnote{See \underline{http://www.stats.ox.ac.uk/$\sim$teh/notes.html} for more details about the derivations of \emph{marginal likelihood} and \emph{posterior} computations.}. Thus its computational complexity is mainly determined by two parts. One part needs no repeated updates, which contains the ratio of Gamma terms with $O(dN)$ ($N$ is the total number of the instances) in ML and the prior of covariance matrix with $O(d^3)$ in both ML and PO, the other is the posterior update of the covariance matrix with $O(d^2T)$ ($T$ is the number of iterations) in both ML and PO. Further considering the number of mixture component $K$, the number of groups/classes $J$, the total computational complexity of CD-OSR is roughly about $O(d^3+dN+d^2TKJ)$. Note that $K$ is changing in each iteration.

\section{Experimental Evaluation}
To verify the effectiveness of our CD-OSR framework, we carry out several experiments on the benchmark datasets commonly used in OSR scenario, including LETTER \cite{frey1991letter}, USPS \cite{hull1994database}, PENDIGITS \cite{bilenko2004integrating}, COIL20\cite{nene1996columbia}, and Extended Yale B \cite{georghiades2001few}. As an initial solution towards collective decision for open set recognition, we compare our CD-OSR with the mainstream OSR methods, including the 1-vs-Set machine, W-OSVM\footnote{W-OSVM is the W-SVM model which only uses the one-class SVM CAP model.}, W-SVM, $P_{I}$-SVM and OSNN. Note that the W-SVM and $P_{I}$-SVM are the currently popular algorithms.

Here we mainly focus on the comparisons of the F-measure among these methods since it better emphasizes the distinction between correct positive and negative classifications \cite{Scheirer2014Probability}. The F-measure is defined as a harmonic mean of Precision and Recall
\begin{equation*}
\text{F-measure} = 2\cdot\frac{\text{Precision}\cdot\text{Recall}}{\text{Precision}+\text{Recall}},
\end{equation*}
where
\begin{equation*}
\text{Precision} = \frac{\text{TP}}{\text{TP}+\text{FP}}
\end{equation*}
and
\begin{equation*}
\text{Recall} = \frac{\text{TP}}{\text{TP}+\text{FN}}.
\end{equation*}
TP, FN and FP respectively represent true positive, false negative and false positive of known classes. Note that although the computations of Precision and Recall are only for available known classes, the FN and FP actually also consider the false unknown classes and false known classes by taking the false negative and the false positive into account \cite{junior2017nearest}. Concretely, we use the micro-F-measure \cite{junior2017nearest} as an evaluation metric. The higher the micro-F-measure, the better the performance of an OSR algorithm. For comparison, we also give the recognition accuracy for these algorithms in the Supplementary Material.


In addition, the experimental setups including the experimental protocol and the parameter setting are given in subsection 4.1. subsection 4.2 presents the main experimental results, while the new class discovery function is reported in subsection 4.3.

\subsection{Experimental setup}
\subsubsection{Experimental protocol}
\begin{figure}[!t]
\centering
\includegraphics[width=8cm]{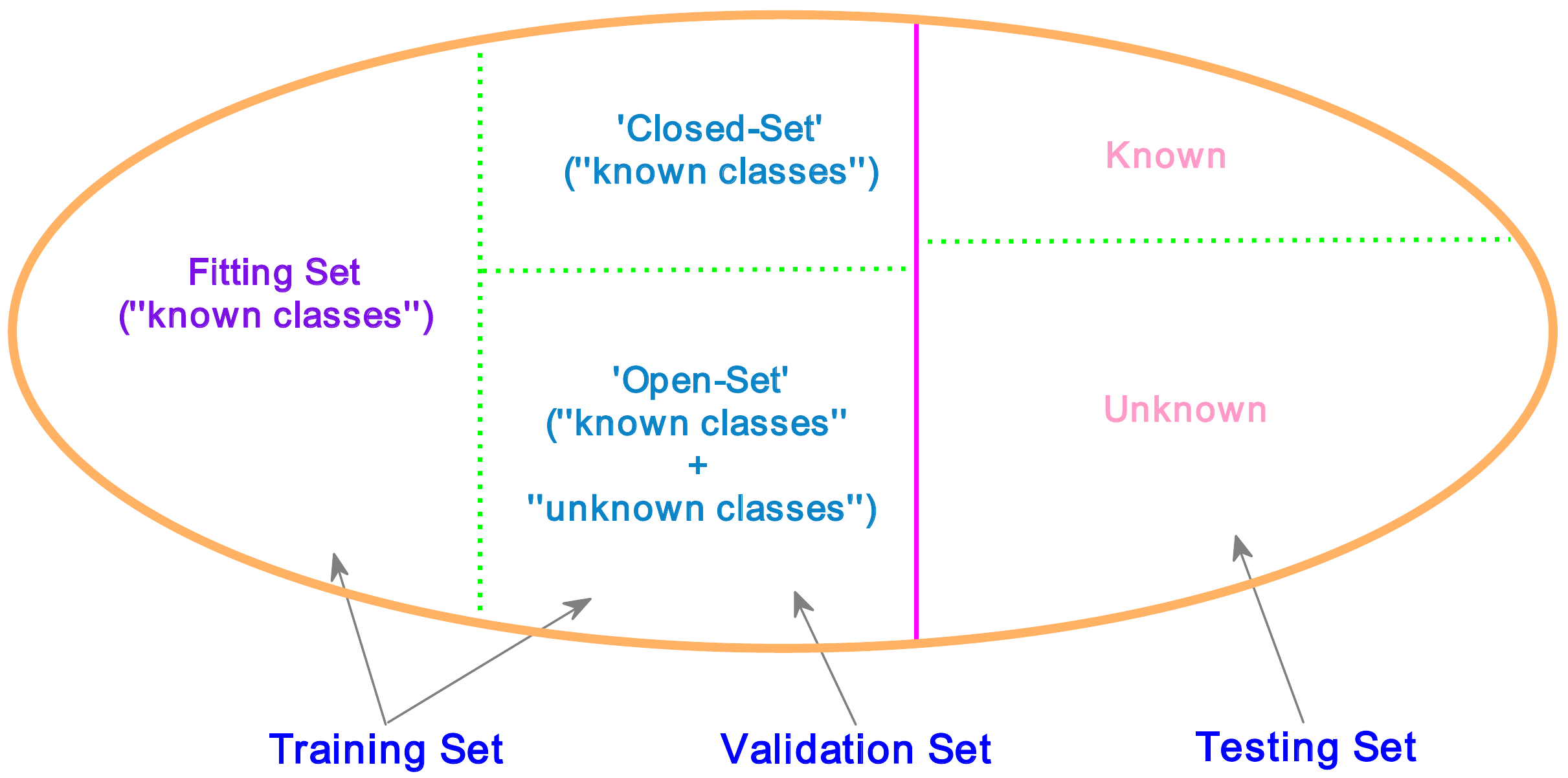}
\caption{Data partitioning. The dataset is first divided into training and testing sets, then the training set is further divided into a fitting set and a validation set containing a 'closed set' simulation and a 'open set' simulation.}
\end{figure}
As described in Section 1, the selection of suitable thresholds for the corresponding OSR methods is difficult and risky due to lacking available information from unknown classes. To mitigate this challenge, similar to \cite{junior2017nearest}, a parameter optimization phase adapted to the OSR scenario is performed to find the better parameters for all methods in this paper. Note that the optimal parameter values are selected based on the tradeoff on F-measure between the simulations of 'Closed-Set' and 'Open-Set' scenarios built in the validation set.

As shown in Fig. 3, the dataset is first divided into training set owning known classes and testing set including known and unknown classes, respectively. Among the classes occurring in training set, half are chosen to act as ''known'' classes in the simulation, the other half as ''unknown'' in the simulation. Thus the training set is divided into a fitting set $\mathcal{F}$ just containing the ''known'' classes and a validation set $\mathcal{V}$ including a  'Closed-Set' simulation and an 'Open-Set' simulation. The 'Closed-Set' simulation only owns the ''known'' classes, while the 'Open-Set' simulation contains all the classes appearing in the training set. Note that in the training phase, all the methods are trained with $\mathcal{F}$ and evaluated on $\mathcal{V}$. Additionally, we give the following experimental protocol. For each experiment in this paper, we
\begin{enumerate}[\IEEEsetlabelwidth{8)}]
\item [1.] randomly select $\Omega$ available classes as known classes for training from the dataset;
\item [2.] randomly choose 60\% of the instances in each of the $\Omega$ selected classes as training set;
\item [3.] select the remaining 40\% of the instances from step 2 and the instances from other classes excluding the $\Omega$ classes as testing set;
\item [4.] randomly select $[(\Omega/2+0.5)]$ classes as ''known'' classes for fitting from the training set, while the remaining classes as ''unknown'' classes for evaluating;
\item [5.] randomly choose 60\% of instances from each ''known'' classes of the training set as fitting instances in $\mathcal{F}$;
\item [6.] select the remaining 40\% of the instances from step 5 as the 'Closed-Set' simulation, while the remaining 40\% of the instances from step 5 and the ones from ''unknown'' classes  in training set as the 'Open-Set' simulation;
\item [7.] train all the models with $\mathcal{F}$ and evaluate them on $\mathcal{V}$, then find the suitable parameters;
\item [8.] evaluate the performance for all the methods with 10 randomized training and testing sets after the parameters of corresponding models are determined.
\end{enumerate}
\textbf{Remark:} while several different randomness in the experiments (e.g., the Gibbs sampling during the inference process, the random division for the dataset and so on), the experimental results in this paper are from the repetition of multiple evaluations based on the corresponding random division for the dataset.

\subsubsection{Parameter setting}
This part details the parameter setting for all of the methods used in this paper. For the 1-vs-Set machine, we use the default setting in the code provided by the authors. For W-OSVM and W-SVM adopting one-vs-rest approach, we fix the threshold $\rho_\tau$ for the one-class SVM CAP model in 0.001 as specified by the authors, while a grid search in $\{10^{-7},10^{-6},...,10^{-1}\}$ is performed for threshold $\rho_R$. Similar to W-SVM, $P_I$-SVM also uses the one-vs-rest approach. Accordingly, a grid search in $\{10^{-7},10^{-6},...,10^{-1}\}$ is  performed for threshold $\rho$ in $P_I$-SVM. As for the related SVM parameters including the W-OSVM, W-SVM and $P_I$-SVM, we perform grid search for $C\in\{2^{-5},2^{-4},...,2^{5}\}$ and $g\ (\text{gamma})\in\{2^{-8},2^{-7},...,2^{3}\}$. Furthermore, the implementation codes including 1-vs-Set machine, W-OSVM, W-SVM and $P_I$-SVM can be found at \emph{https://github.com/ljain2/libsvm-openset}. For OSNN, only the threshold $\sigma$ needs to be optimized, and we adopt the same strategy described in \cite{junior2017nearest}. Please note that in this paper, once the thresholds of these methods are determined in training, their values will no longer change in testing, since we usually know nothing about unknown classes.

For CD-OSR, we have two learning phase. In the training phase, our goal is to get the the appropriate initialization parameters. Towards this goal, we model each known class in the fitting set $\mathcal{F}$ using the Gaussian mixture model. Each component in the mixture model is associated with a Gaussian distribution with the mean vector $\mu_{jt}$ and  covariance matrix $\Sigma_{jt}$, i.e., $\psi_{jt}=\{\mu_{jt},\Sigma_{jt}\}$. For the base distribution $H$, we define a conjugate prior, i.e., Gaussian-Wishart distribution
\begin{equation}
H = p(\mu,\Sigma|\mu_0,\beta,\Sigma_0,\nu) = \mathcal{N}(\mu|\mu_0,(\beta\Sigma)^{-1})\mathcal{W}(\Sigma|\Sigma_0,\nu),
\end{equation}
where $\mu_0$ is the prior mean, $\beta$ is a scaling constant controlling the deviation of the mean vectors of mixture components from the prior mean. $\Sigma_0$ denotes the prior covariance matrix, and $\nu$ is the number of \emph{degrees of freedom} of the distribution. In order to confirm the validity of our learning framework, we do not take an overly complicated means to select the initialization parameters in the CD-OSR. In contrast, we here let $\mu_0$ simply equal the mean of all the instances in $\mathcal{F}$\footnote{In testing, $\mu_0$ is simply set to the mean of all the instances in Training set.}, $\beta$ equal 1, and $\nu$ be selected by performing a
grid search from the set $\{d,d+1,...,d+20\}$. Furthermore, $\Sigma_0$ is set as follows
\begin{equation}
\Sigma_0 = \varsigma\times\frac{\sum_{j=1}^{J-1}(n_j-1)\Sigma_j}{n-(J-1)}
\end{equation}
where the $\varsigma$ is a scaling constant and also obtained by performing a grid search from the set $\{0.00001,0.0001,0.001,0.01,0.1,0.2,...,1\}$. $J-1$ represents the number of known classes in $\mathcal{F}$\footnote{There are a total of $J$ groups under the CD-OSR framework, where the former $J-1$ groups represent the known classes in $\mathcal{F}$ and the $J$-th group represents the  validation set $\mathcal{V}$.}. $n$ is the total number of the instances in $\mathcal{F}$. The second term on the right side of (13) denotes the common pooled covariance matrix of the known classes \cite{greene1989partially}.  Moreover, for the base distributions $H$ and $G$, the concentration parameters are given by the vague gamma priors \cite{Escobar1995Bayesian}. Specifically, we set $\gamma\sim\text{Gamma}(100,1)$ and $\alpha_0\sim\text{Gamma}(10,1)$ to ensure enough subclasses used to represent each known class, while reducing the sharing of subclasses between the different known classes. The maximum number of iterations of CD-OSR ($T$) is set to 30, while the initial number of mixture components (InS) is set to 30. Additionally, the $\varepsilon$ is empirically set to $0.01$, where subsection 4.2.3 details the reason of this setting.

After the training phase, we will obtain the appropriate values of initialization parameters for CD-OSR. Fixing these parameters' values, we only need to respectively replace fitting set $\mathcal{F}$ and validation $\mathcal{V}$ with the training set and testing set, then repeat 10 rounds of the co-clustering process to obtain the final experimental evaluation.

\subsection{Performance Evaluation}
This subsection mainly contains three parts. subsection 4.2.1 reports the F-measure comparisons among our CD-OSR with the 1-vs-Set, W-OSVM, W-SVM, $P_I$-SVM, and OSNN. subsection 4.2.2 shows the influence of the batch size of the testing data on performance, while the influence on parameter $\varepsilon$ is discussed in subsection 4.2.3.

\subsubsection{Comparisons on F-measure}
\textbf{Results on LETTER}: The LETTER dataset has a total of 20000 instances from 26 classes, where every instance owns 16 features.  To recast Letter dataset as a dataset for open set problem, we randomly select 10 available classes as known classes for training, and vary openness by adding a subset of the remaining 16 classes. Fig. 4 shows the average F-measure results on this dataset. With the openness less than about 12\%, the performance of our CD-OSR is comparable to the W-SVM and $P_I$-SVM. However, it is almost significantly higher than the other five methods used in this paper when the openness is larger than about 12\%. Furthermore, the changing trend of F-measure in CD-OSR is also relatively stable when varying the openness.

\begin{figure}[ht]
\centering
\includegraphics[width=8.8cm]{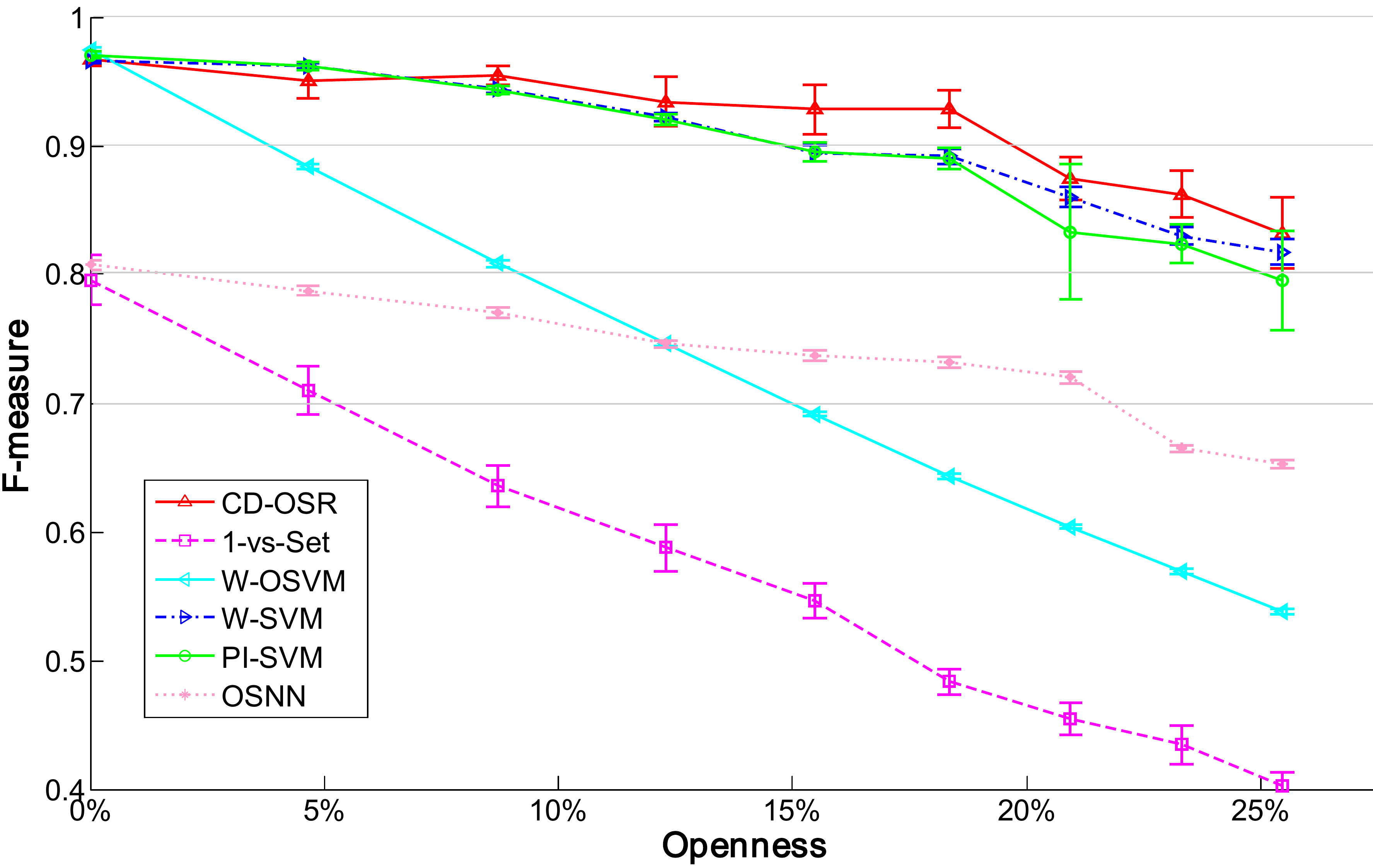}
\caption{F-measure for multi-class open set recognition on LETTER dataset. Error bars reflect the standard deviation.}
\end{figure}

\begin{figure}[ht]
\centering
\includegraphics[width=8.8cm]{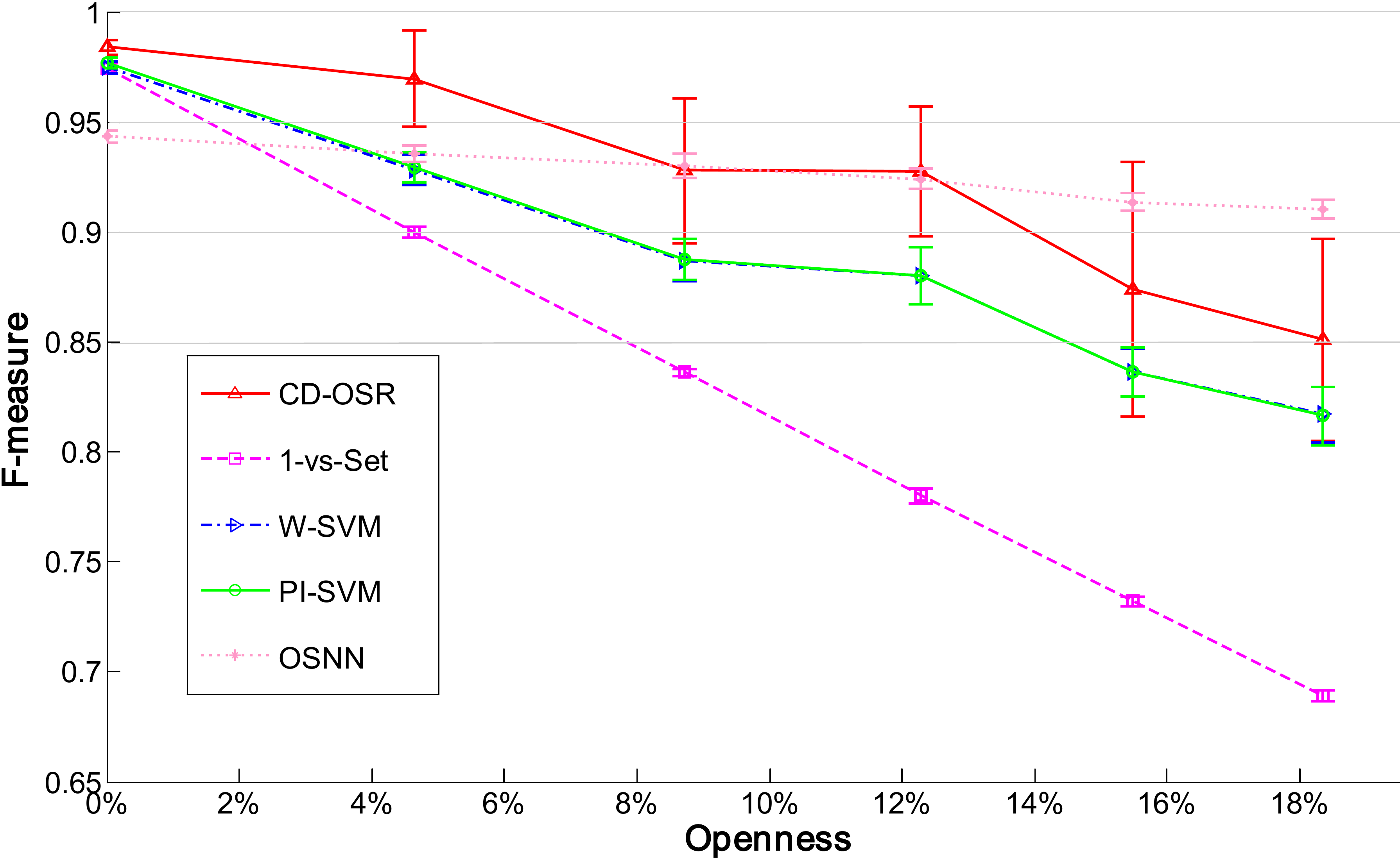}
\caption{F-measure for multi-class open set recognition on USPS dataset, and the performance of W-OSVM is not shown due to its poor F-measure. Error bars reflect the standard deviation.}
\end{figure}

\begin{figure}[ht]
\centering
\includegraphics[width=8.8cm]{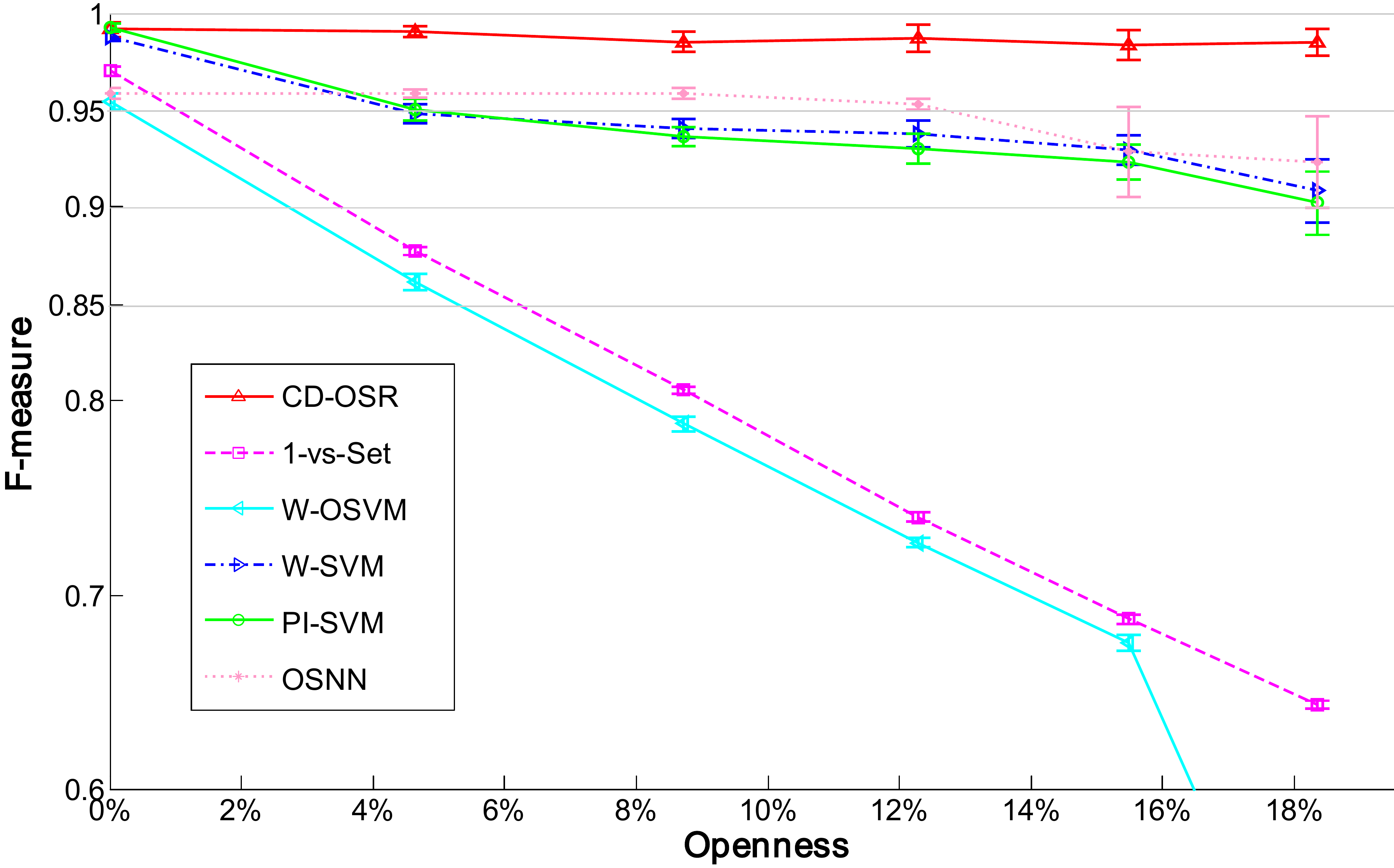}
\caption{F-measure for multi-class open set recognition on PENDIGITS dataset. Error bars reflect the standard deviation.}
\end{figure}

\textbf{Results on USPS}: The USPS dataset has a total of 7291 instances from 10 classes, where every instance owns 256 features. In this paper, principal component analysis (PCA) is used to project instance space into 39 dimensional subspace, retaining 95\% of the instances' information. Similar to the operation of LETTER dataset, we randomly select 5 available classes as known classes for training, and vary openness by adding a subset of the remaining 5 classes. Fig. 5 shows the average F-measure results on this dataset. As can be seen from Fig. 5, our CD-OSR obtains much higher performance than the 1-vs-Set, W-SVM and $P_I$-SVM with increasing the openness. Although the OSNN obtains the higher performance than the CD-OSR when the openness is larger than about 12\%, its performance is much lower than our method when openness is less than 8\%, especially when the openness equals zero. Furthermore, compared to the other methods, the changing trend of F-measure in OSNN is most stable, followed by our CD-OSR, W-SVM and $P_I$-SVM. Note that the performance of W-OSVM is not shown in Fig. 5 due to its poor F-measure.

\textbf{Results on PENDIGITS}: The PENDIGITS dataset has a total of 10992 instances from 10 classes, where every instance owns 16 features. Similar to the operation of USPS dataset, we randomly select 5 available classes as known classes for training, and vary openness by adding a subset of the remaining 5 classes. Fig. 6 shows the average F-measure results on this dataset. As can be seen from Fig. 6, our CD-OSR obtains much higher performance than the other methods as the openness increases. Simultaneously, the performance of our CD-OSR is almost unchanged when varying the openness.

\begin{figure}[ht]
\centering
\includegraphics[width=8.8cm]{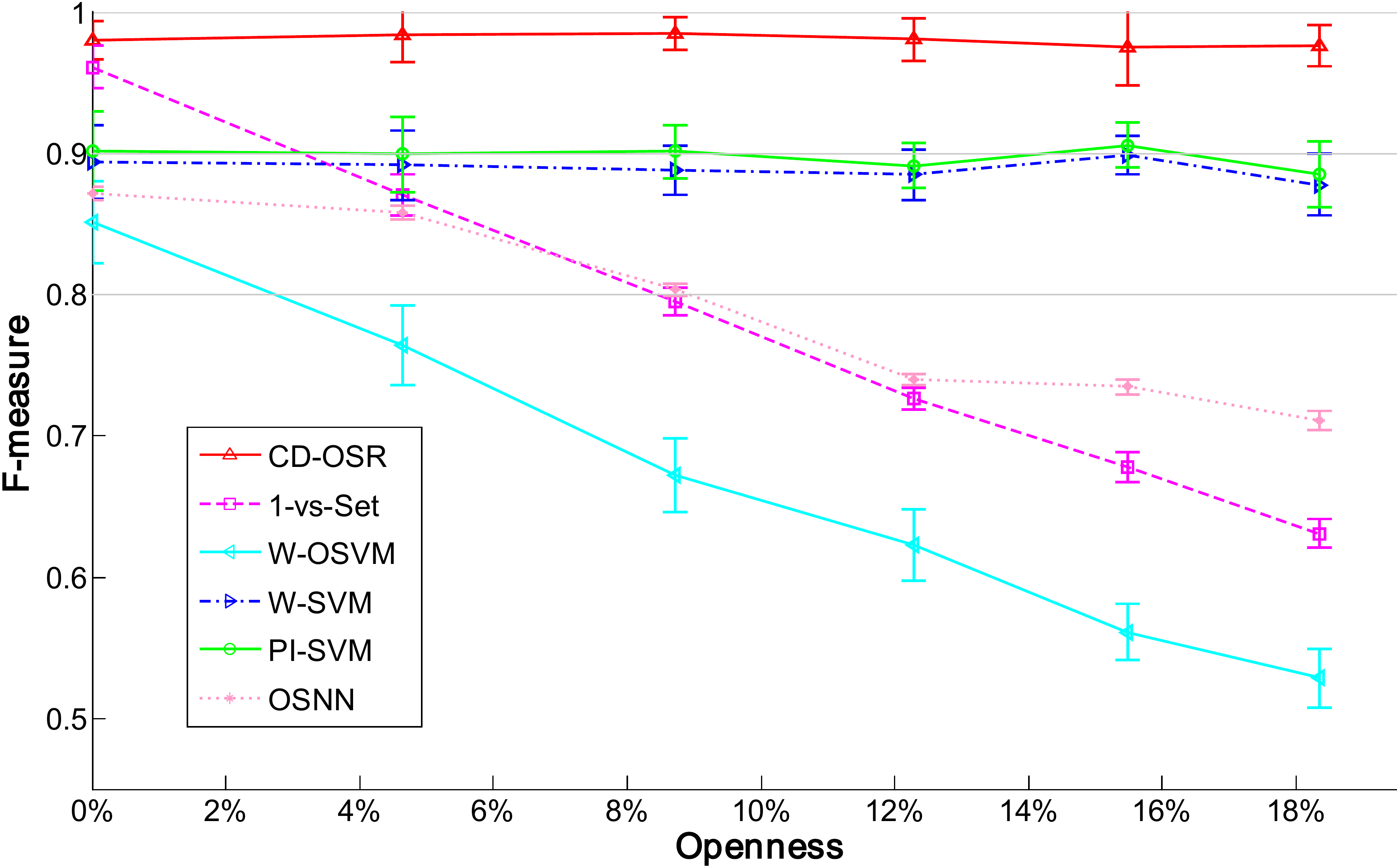}
\caption{F-measure for multi-class open set recognition on COIL20 dataset. Error bars reflect the standard deviation.}
\end{figure}

\begin{figure}[ht]
\centering
\includegraphics[width=8.8cm]{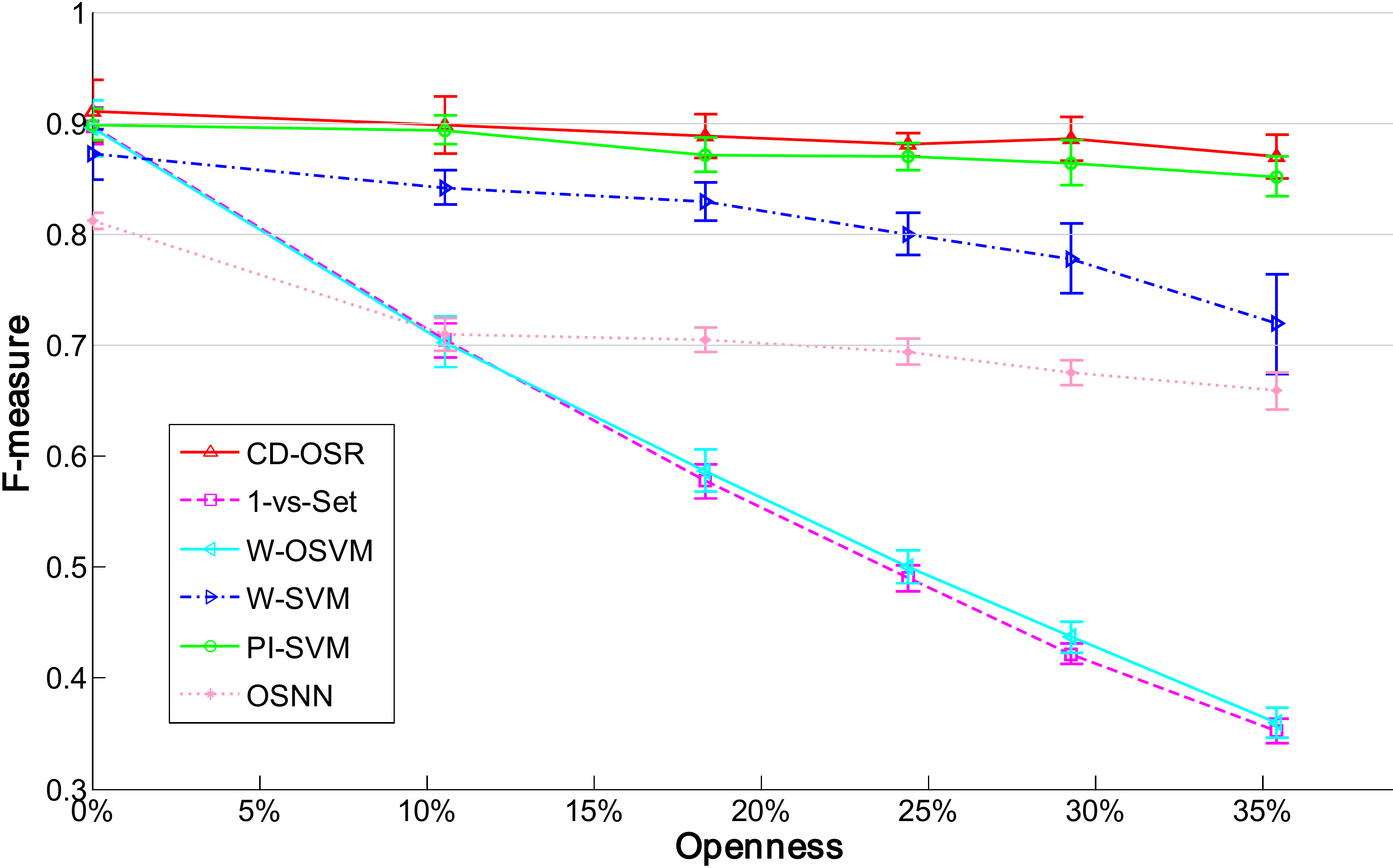}
\caption{F-measure for multi-class open set recognition on Extended Yale B dataset. Error bars reflect the standard deviation.}
\end{figure}

\textbf{Results on COIL20}: The COIL20 dataset has a total of 1440 gray images from 20 objects [48]. Each image is down-sampled to $16\times16$, i.e., the input dimension is 256. We further reduce it to 55 dimensions (retaining 95\% of the instances' information) using PCA technique. Then we randomly select 10 objects as known classes, and vary openness by adding a subset of the remaining 10 objects. As shown in Fig. 7, our CD-OSR is significantly better than the other five algorithms: CD-OSR not only achieves the best F-measure performance but also has the better stability with increasing openness.

\textbf{Results on Extended Yale B}: The Extended Yale B (YALEB) dataset has a total of 2414 frontal-face images from 38 individuals. Each class has around 64 images. The images are cropped and normalized to $32\times32$. Similar to COIL20, we reduce the input dimension to 69 using PCA. Fig. 8 shows the F-measure performance on this dataset. As can been seen from Fig. 8, both our CD-OSR and $P_I$-SVM gain significant advantages over other methods. Furthermore, with the openness increasing (about openness $>15\%$), CD-OSR starts to be slightly better than $P_I$-SVM. Besides, though W-SVM and OSNN also have good stability with increasing openness, their F-measure performances are not satisfactory.

\textbf{Remark:} From the experimental results reported above,  we can find that the classification performance of our CD-OSR is significantly improved compared to other existing OSR algorithms. However, what we still want to emphasize is that the CD-OSR currently does not make full use of the information from the known class labels. More precisely, it just uses this kind of information to assign the training data to different groups, while the discriminative information from these labels actually is not fully utilized. Nevertheless, CD-OSR still achieves at least comparable classification performance than other existing OSR methods making full use of label information like W-SVM, $P_I$-SVM, and so forth. Additionally, the above experimental results are also shown in tables of the Supplementary Material to further demonstrate the superiority of our CD-OSR.

\begin{figure}[!t]
\centering
\subfloat[]{\includegraphics[width=1.77in]{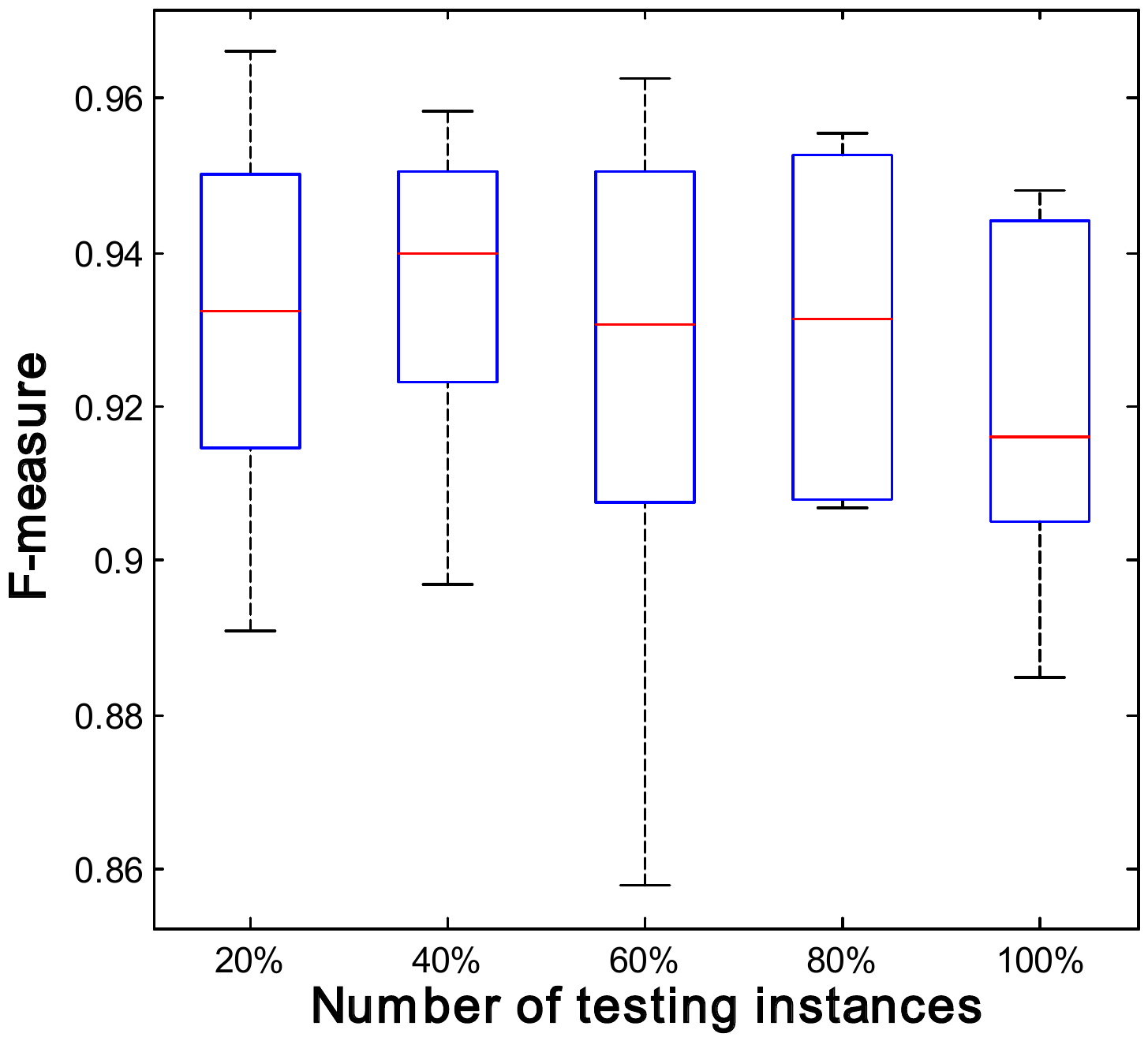}%
\label{fig_first_case}}
\hfil
\subfloat[]{\includegraphics[width=1.72in]{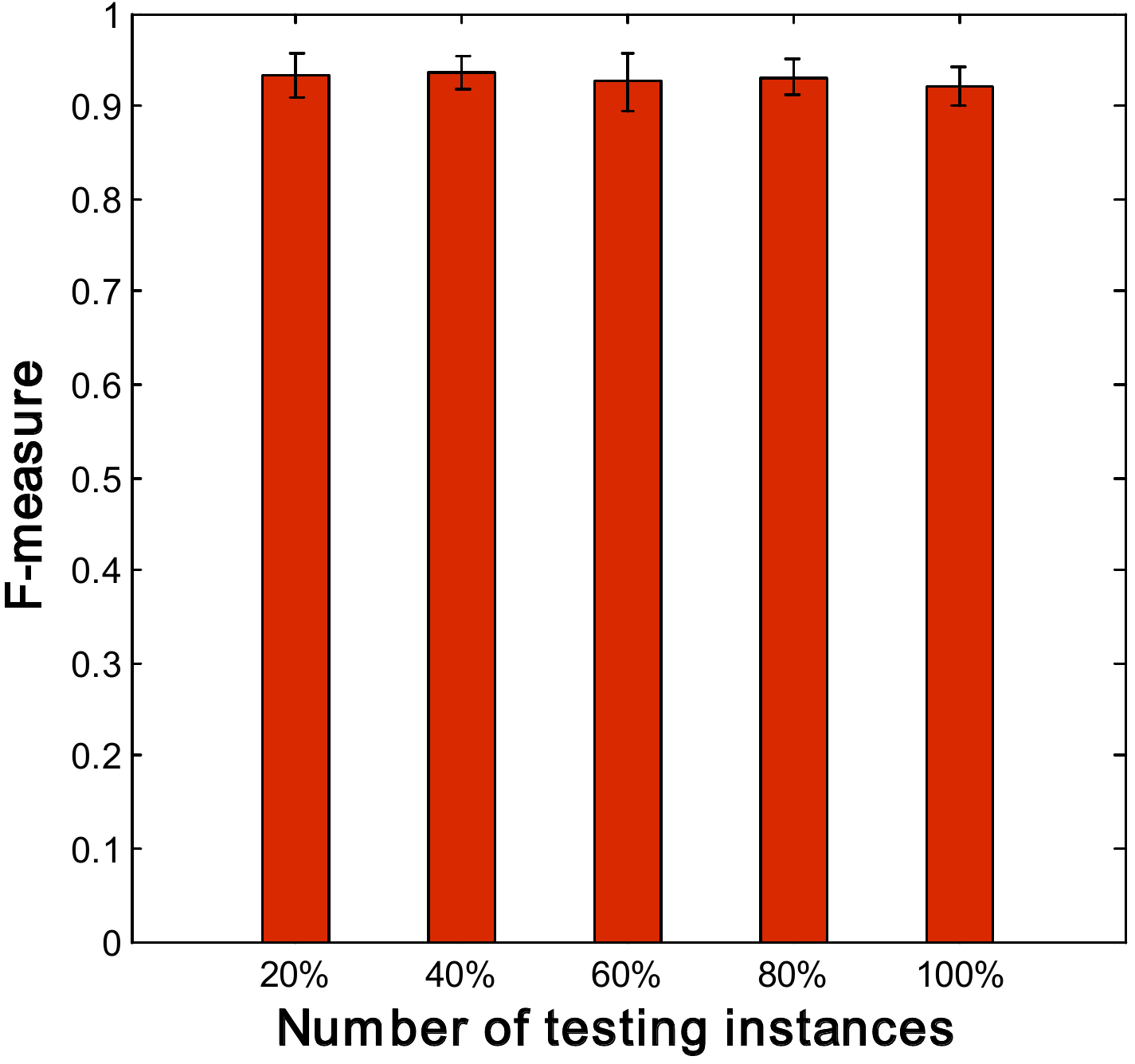}%
\label{fig_second_case}}
\caption{The F-measure on LETTER  dataset when openness = 18.35\%. (a) denotes the boxplot graph for the different number of testing instances, while (b) represents the corresponding errorbar graph, where error bars reflect the standard deviation.}
\label{fig_sim}
\end{figure}

\begin{figure}[!t]
\centering
\subfloat[]{\includegraphics[width=1.77in]{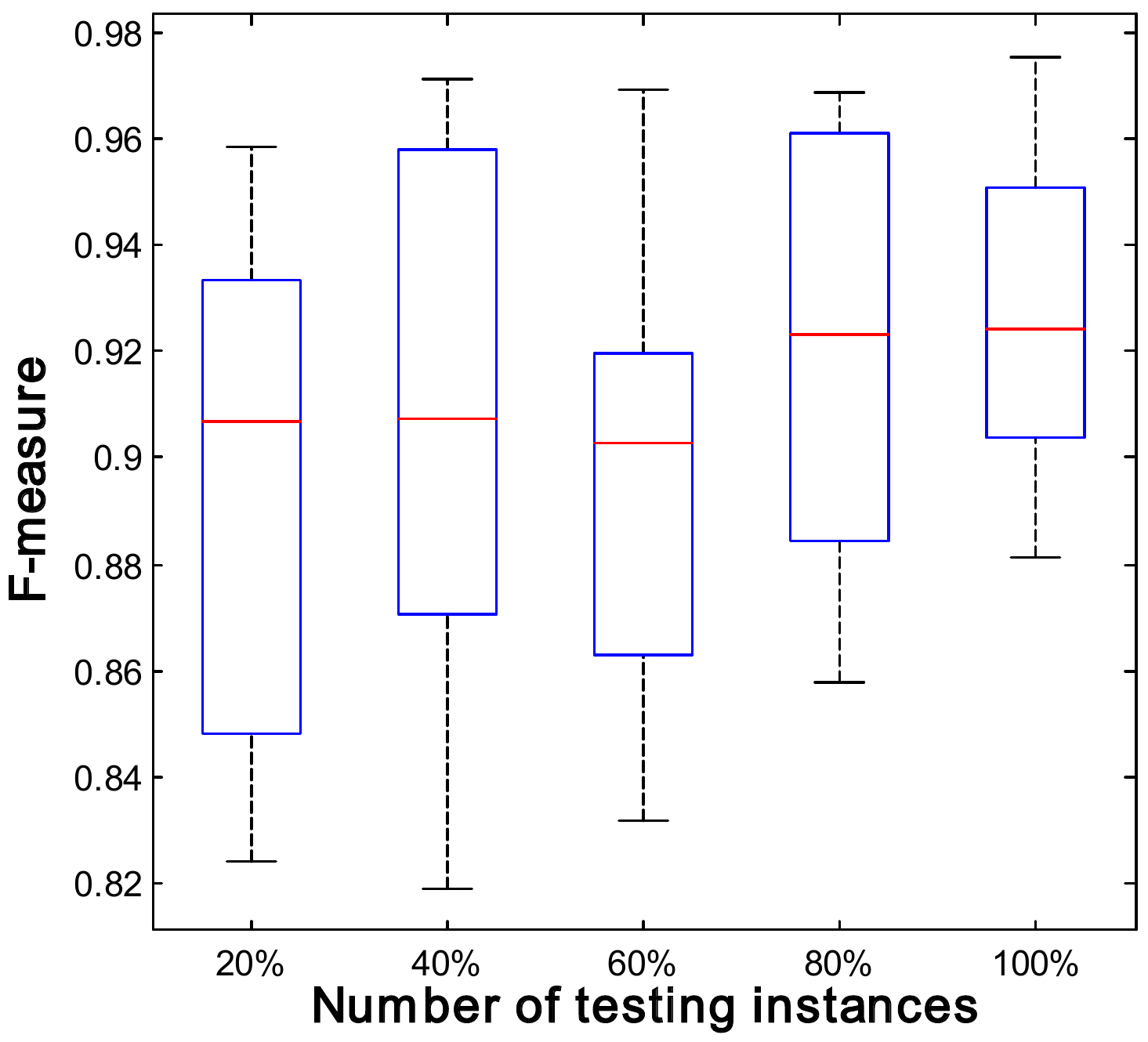}%
\label{fig_first_case}}
\hfil
\subfloat[]{\includegraphics[width=1.72in]{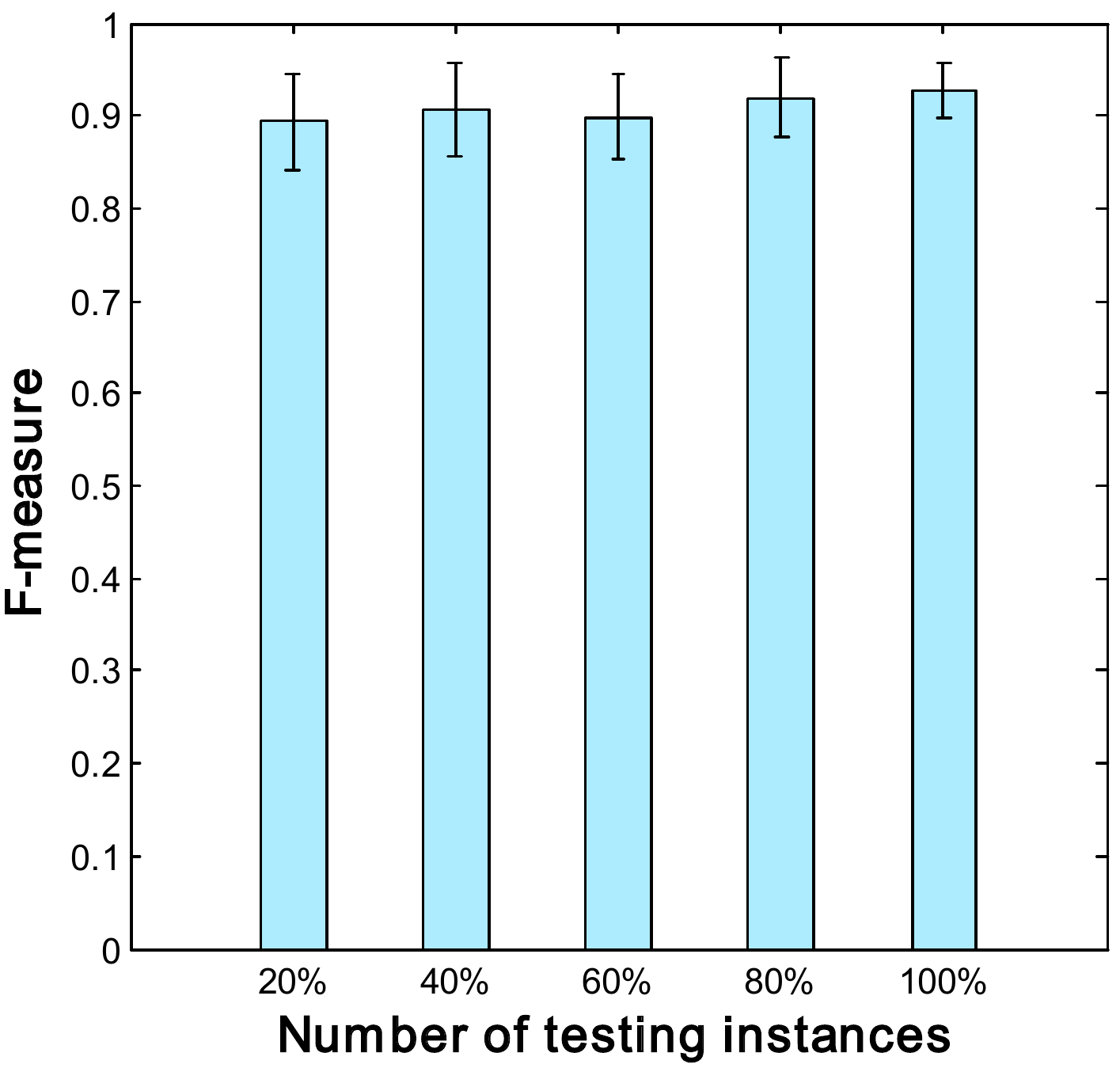}%
\label{fig_second_case}}
\caption{The F-measure on USPS dataset when openness = 12.29\%. (a) denotes the boxplot graph for the different number of testing instances, while (b) represents the corresponding errorbar graph, where error bars reflect the standard deviation.}
\label{fig_sim}
\end{figure}

\begin{figure}[!t]
\centering
\subfloat[]{\includegraphics[width=1.77in]{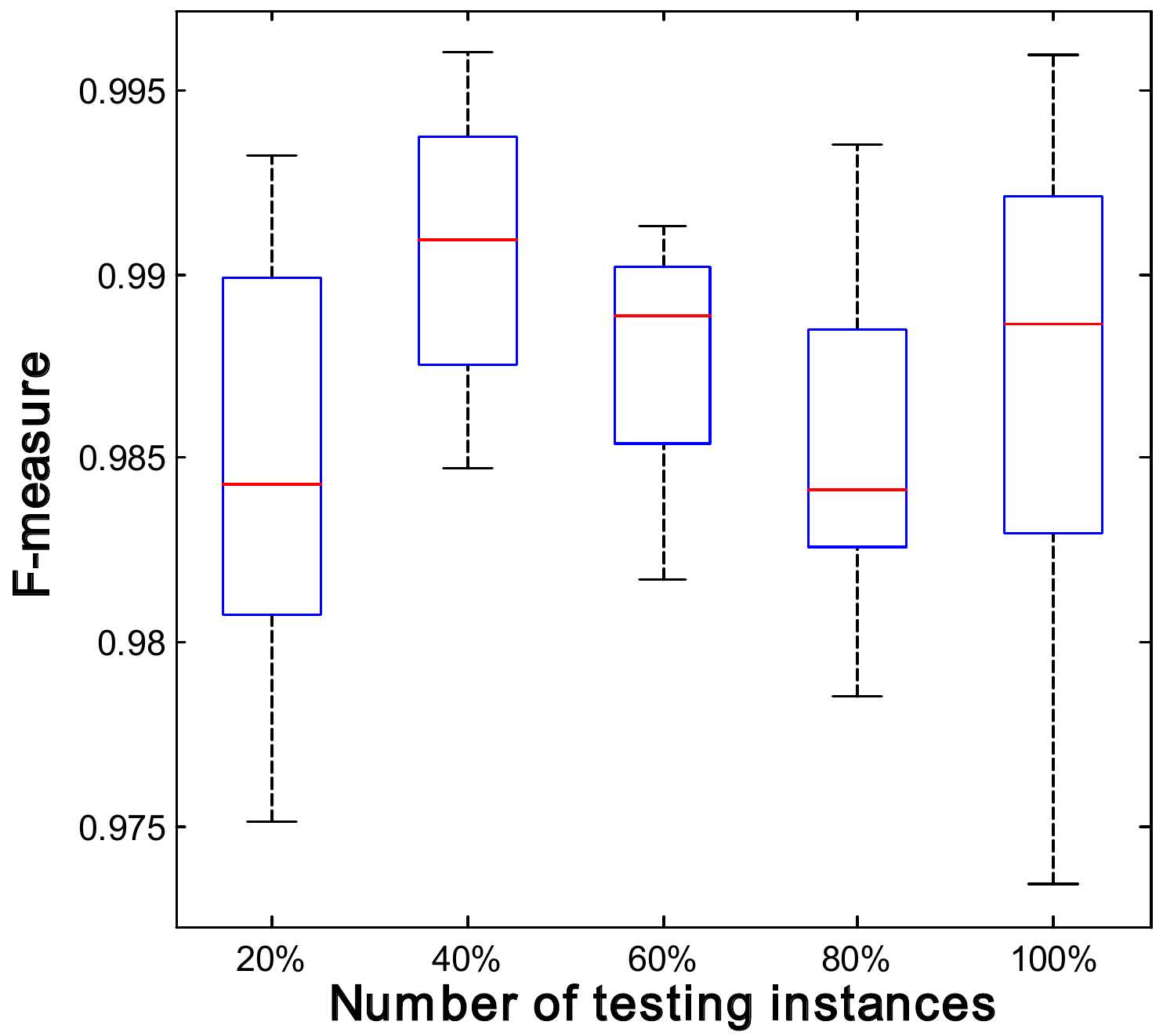}%
\label{fig_first_case}}
\hfil
\subfloat[]{\includegraphics[width=1.72in]{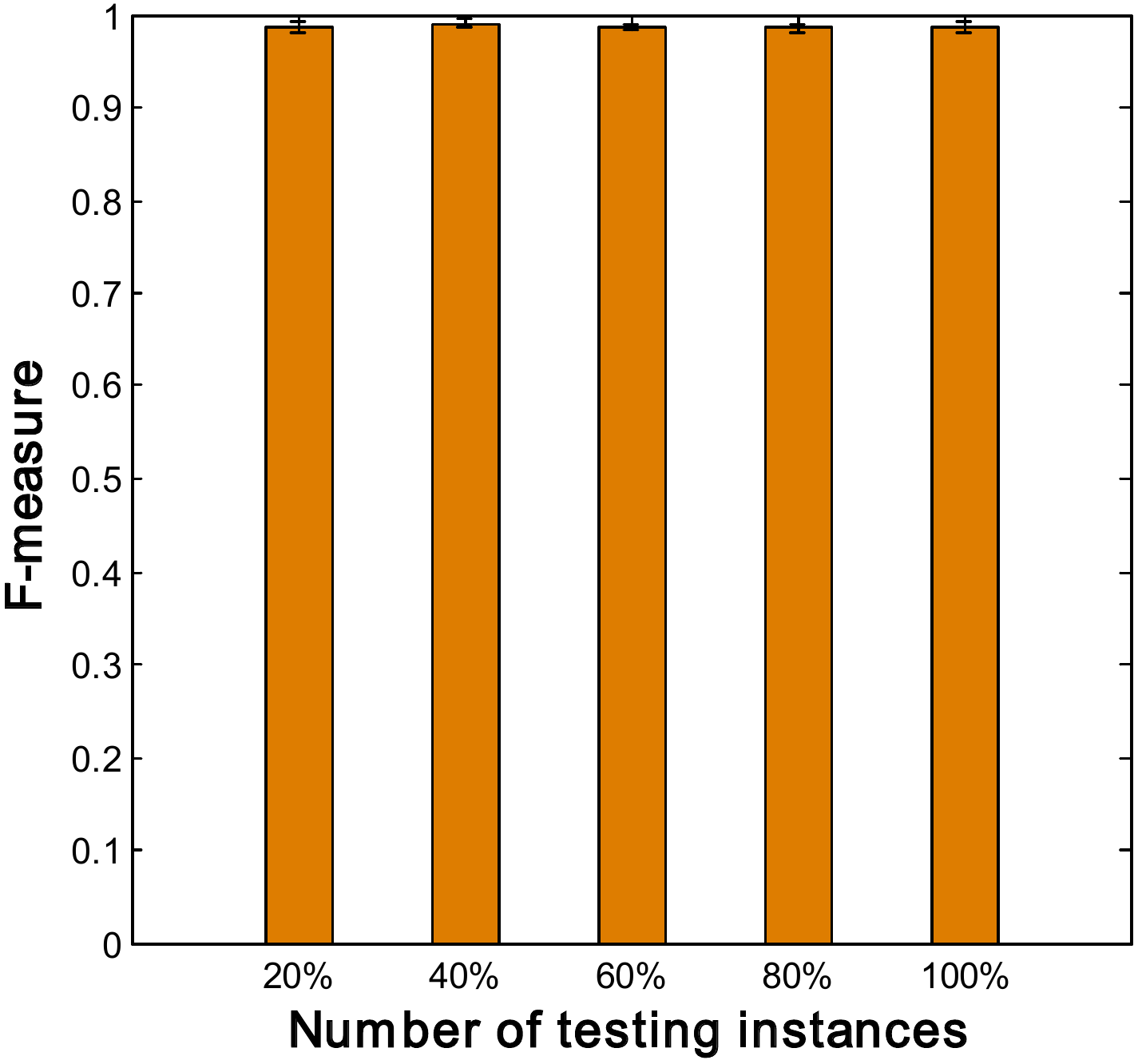}%
\label{fig_second_case}}
\caption{The F-measure on PENDIGITS dataset when openness = 12.29\%. (a) denotes the boxplot graph for the different number of testing instances, while (b) represents the corresponding errorbar graph, where error bars reflect the standard deviation.}
\label{fig_sim}
\end{figure}

\begin{figure}[!t]
\centering
\subfloat[]{\includegraphics[width=1.67in]{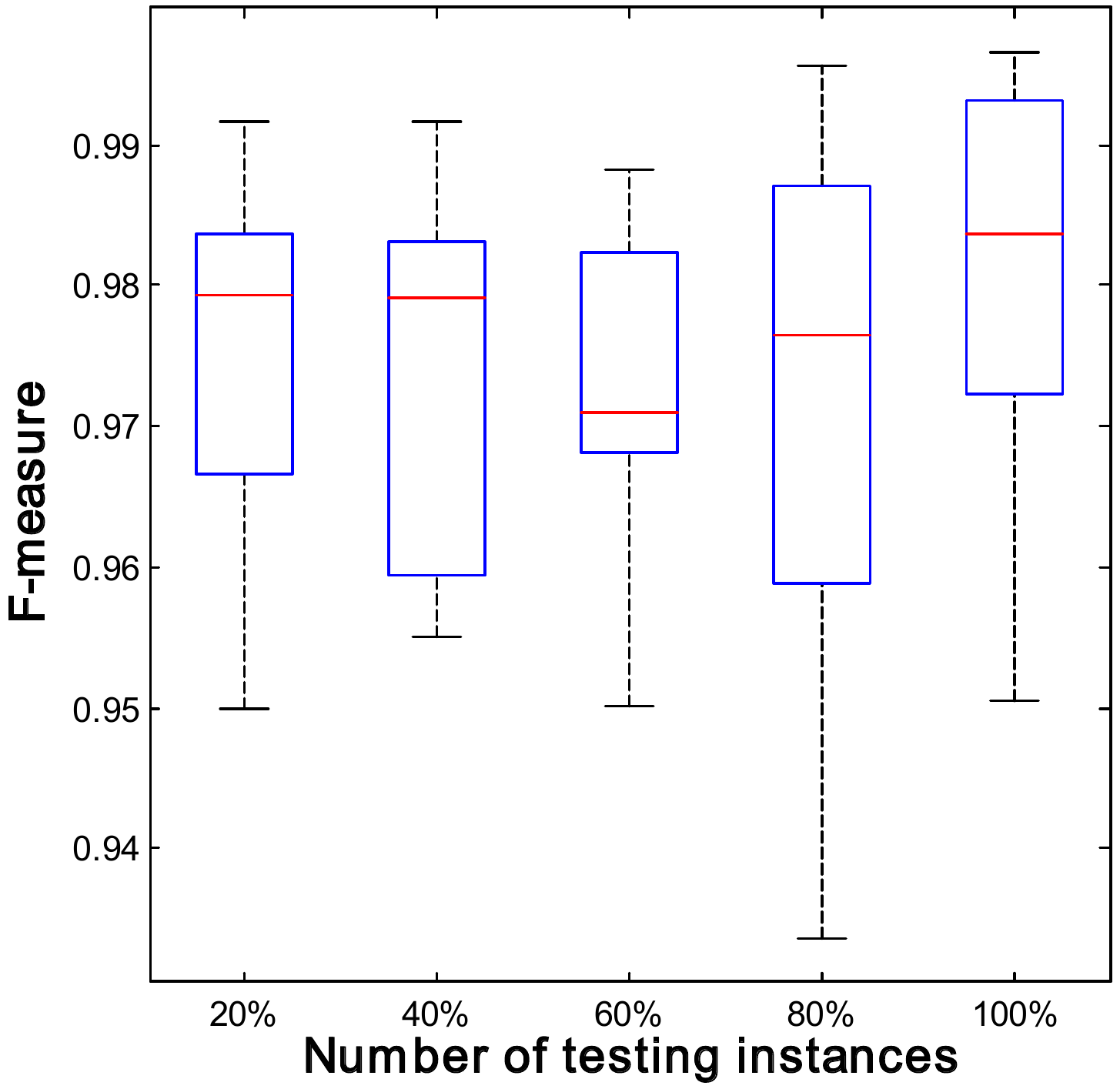}%
\label{fig_first_case}}
\hfil
\subfloat[]{\includegraphics[width=1.63in]{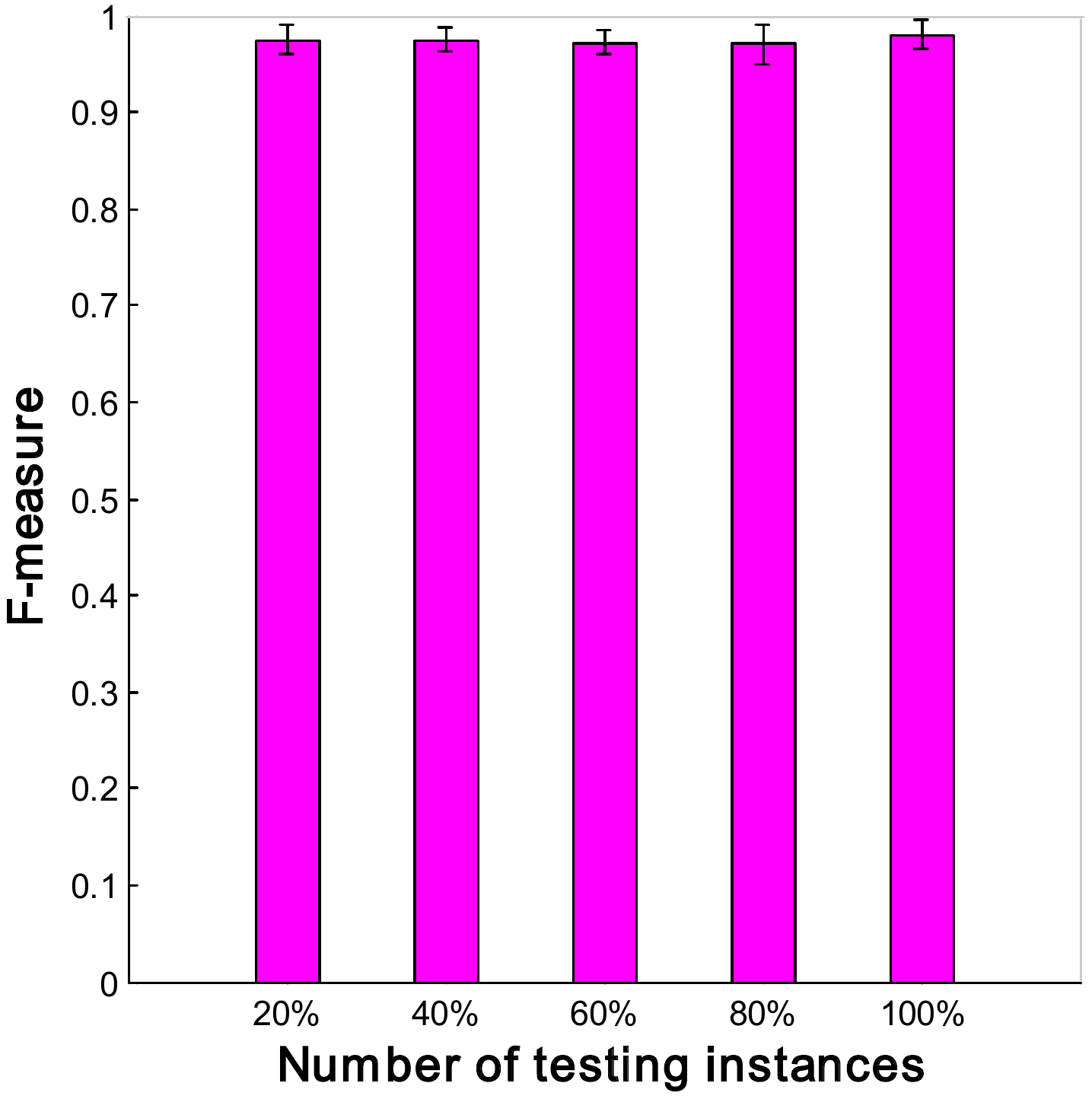}%
\label{fig_second_case}}
\caption{The F-measure on COIL20 dataset when openness = 12.29\%. (a) denotes the boxplot graph for the different number of testing instances, while (b) represents the corresponding errorbar graph, where error bars reflect the standard deviation.}
\label{fig_sim}
\end{figure}

\begin{figure}[!t]
\centering
\subfloat[]{\includegraphics[width=1.68in]{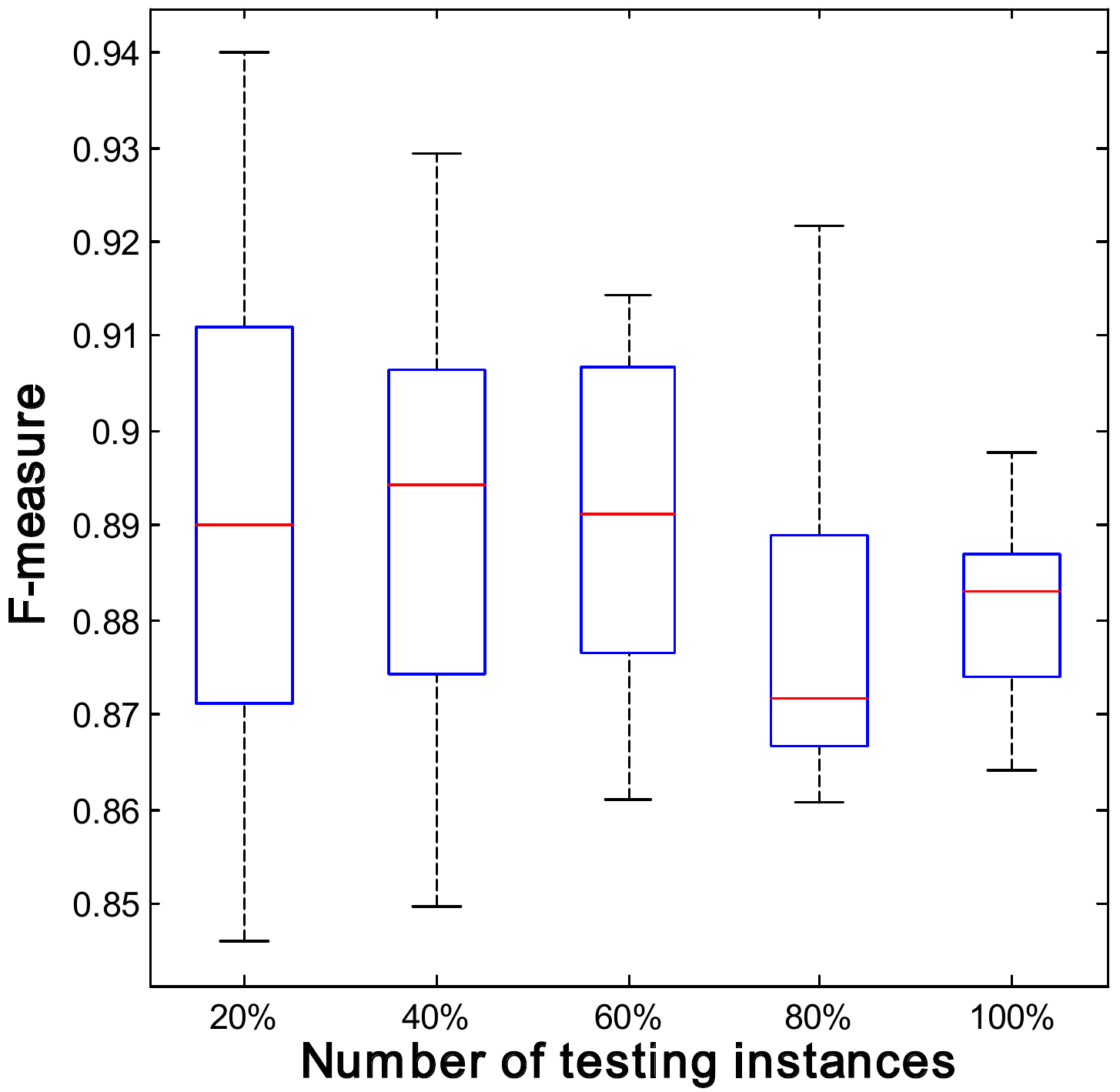}%
\label{fig_first_case}}
\hfil
\subfloat[]{\includegraphics[width=1.64in]{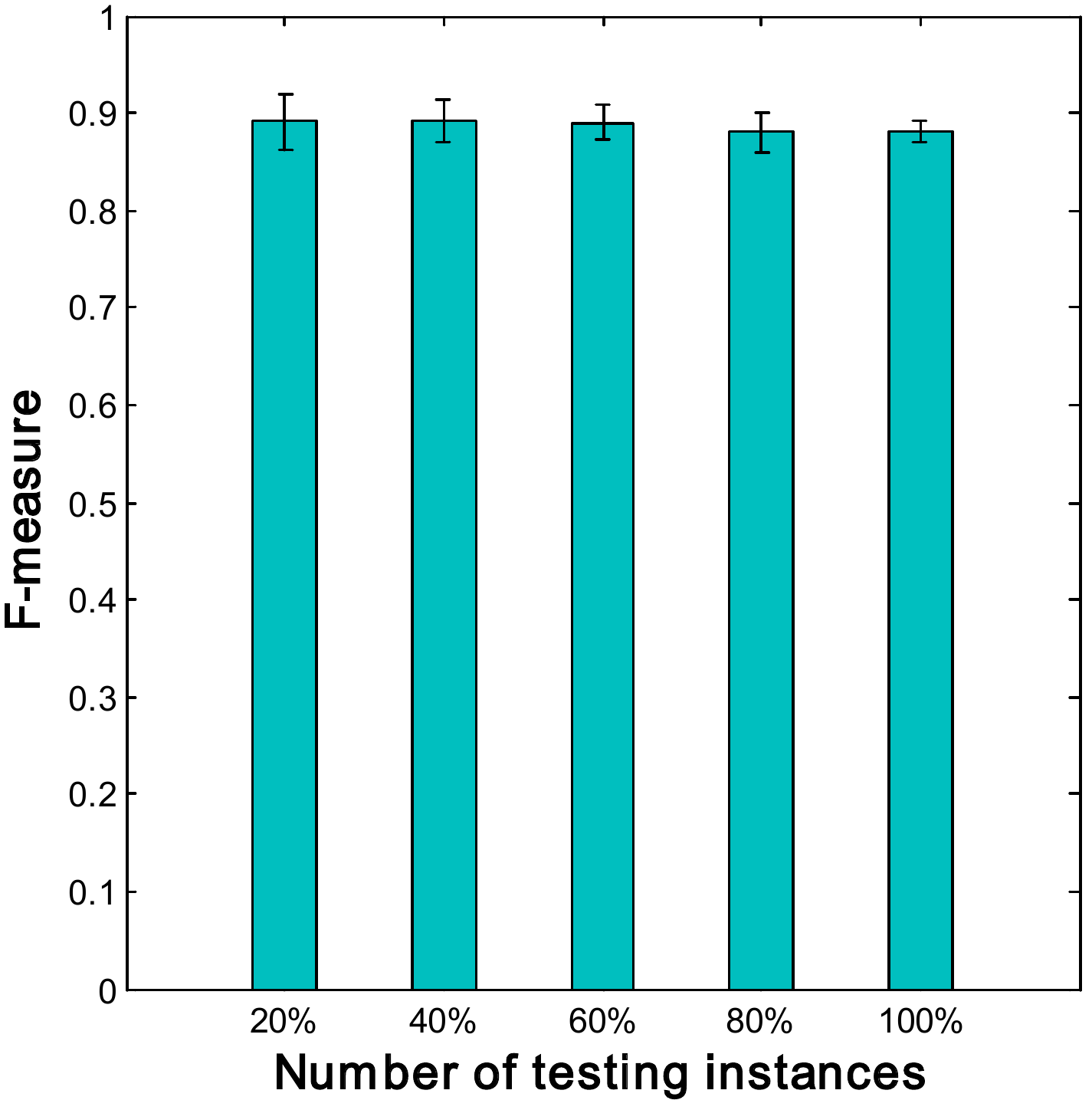}%
\label{fig_second_case}}
\caption{The F-measure on Extended Yale B dataset when openness = 18.35\%. (a) denotes the boxplot graph for the different number of testing instances, while (b) represents the corresponding errorbar graph, where error bars reflect the standard deviation.}
\label{fig_sim}
\end{figure}

\subsubsection{The Influence of the Batch Size on Performance}
Since our CD-OSR adopts the collective/batch decision strategy, meaning it can address the data in batch. Then a natural problem is that whether the size of batch for testing data has an influence on CD-OSR's performance. To explore this problem, we conduct the following experiments.

For each dataset in our experiments, we choose a medium openness: 18.35\% for LETTER (10 unknown classes), 12.29\% for USPS (3 unknown classes), 12.29\% for PENDIGITS (3 unknown classes), 12.29\% for COIL20 (6 unknown classes), and 18.35\% for YALEB (10 unknown classes) then vary the size of the batch by changing the number of testing instances. Specifically, we randomly select 20\%, 40\%, 60\%, 80\%, 100\% of the whole testing set for each dataset, then repeat 10 times of the co-clustering process to obtain the final experimental evaluation. Fig. 9--Fig. 13 show the performance of F-measure on these datasets. The (a) in these figures denotes the boxplot graph for the different number of testing instances, while (b) represents the corresponding errorbar graph, where error bars reflect the standard deviation. From these experimental results, we can find that the batch size of the testing instances has almost no significant influence on the performance of CD-OSR. Therefore, we can flexibly set the batch size according to the needs of the tasks.

\subsubsection{The Influence of parameter \bm{$\varepsilon$} on Performance}
Different from the thresholds in existing OSR methods (such as W-SVM, $P_I$-SVM and OSNN, etc.) where they are essentially used to determine the decision boundary between known and unknown classes, the parameter $\varepsilon$ here determines whether a subclass in certain class should be removed or not for avoiding the overfitting. As discussed in Section 1, existing OSR methods obtain such thresholds only based on the known class knowledge. However, this is actually very risky since it is usually agnostic where an instance of an unknown class appears. In contrast, the selection of $\varepsilon$ is relatively easier due to the intuition that both known and unknown classes get equal-treatment. In other words, no matter it is a known class or an unknown class, if any of its subclasses contains just very few instances, the subclass intuitively should  be removed to avoid overfitting.

In order to verify such an intuition, we perform sensitivity experiments on $\varepsilon$ in both closed and open set scenarios. For the
open set scenario, we adopt the same openness setting in subsection 4.2.2, i.e., 18.35\% for LETTER, 12.29\% for USPS, 12.29\% for PENDIGITS, 12.29\% for COIL20, and 18.35\% for YALEB. The candidate set of $\varepsilon$ is \{0, 0.00001, 0.0001, 0.001, 0.01, 0.1\}. As shown in Fig. 14(a-b), the small $\varepsilon$ (about $\leq0.01$) usually leads to satisfactory classification performance for most datasets (4/5), while the change in $\varepsilon$ value does not generally have a dramatic influence on performance. Of course there are also exceptions: the USPS's performance curve has a sharp drop when $\varepsilon<0.01$. \textbf{Even though so}, we can still observe that the trends of performance curves on all datasets are \textbf{nearly consistent} in both closed set and open set scenarios as the $\varepsilon$ value varies, corroborating our intuition. More interestingly, our CD-OSR obtains better performances at $\varepsilon=0.01$ on all datasets, which experimentally indicates the generality for the $\varepsilon$'s selection to great extent. Therefore, $\varepsilon$ can reasonably be set to 0.01 in all experiments of our paper, as described in subsection 4.1.2.

\begin{figure*}[!t]
\centering
\subfloat[]{\includegraphics[width=3.56in]{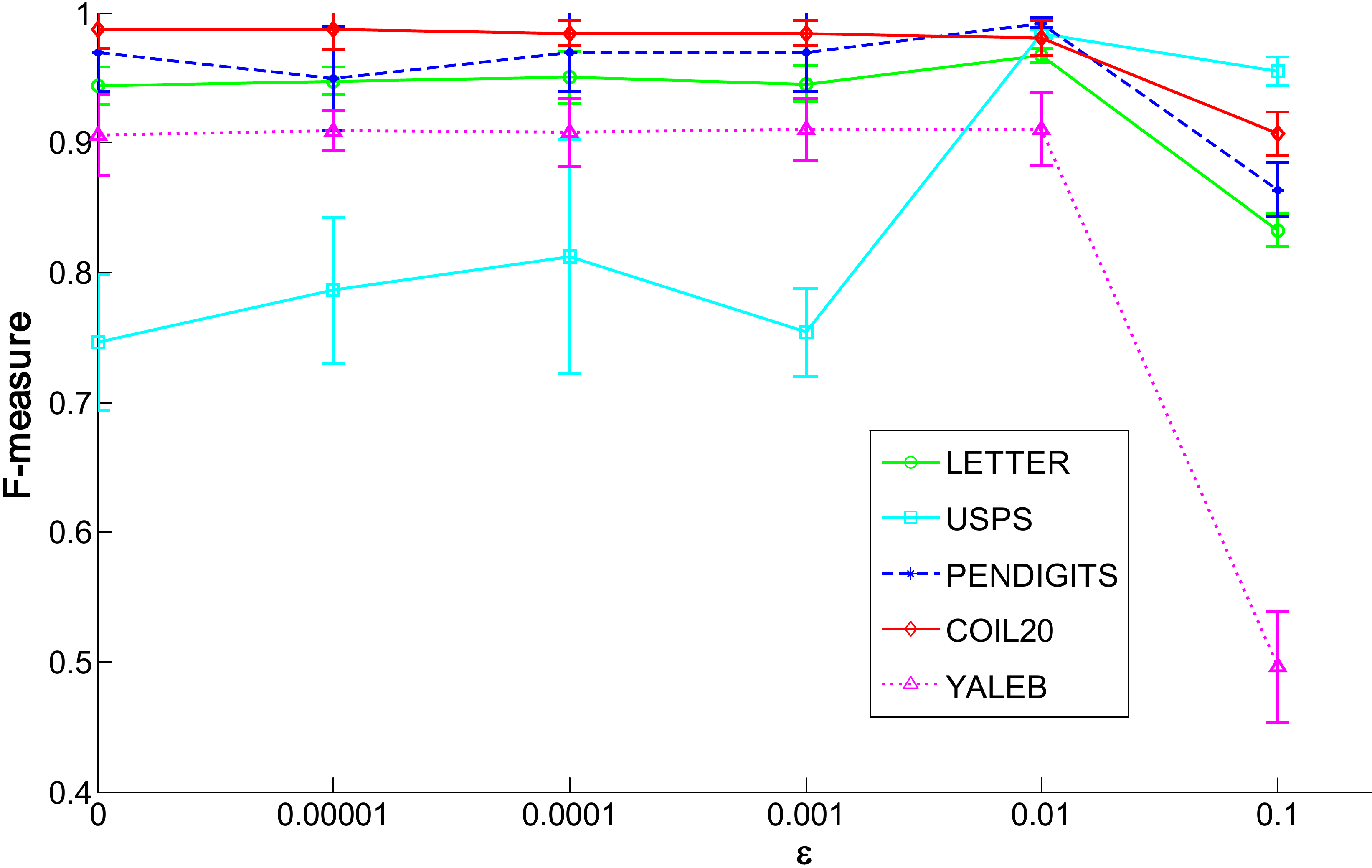}%
\label{fig_first_case}}
\hfil
\subfloat[]{\includegraphics[width=3.55in]{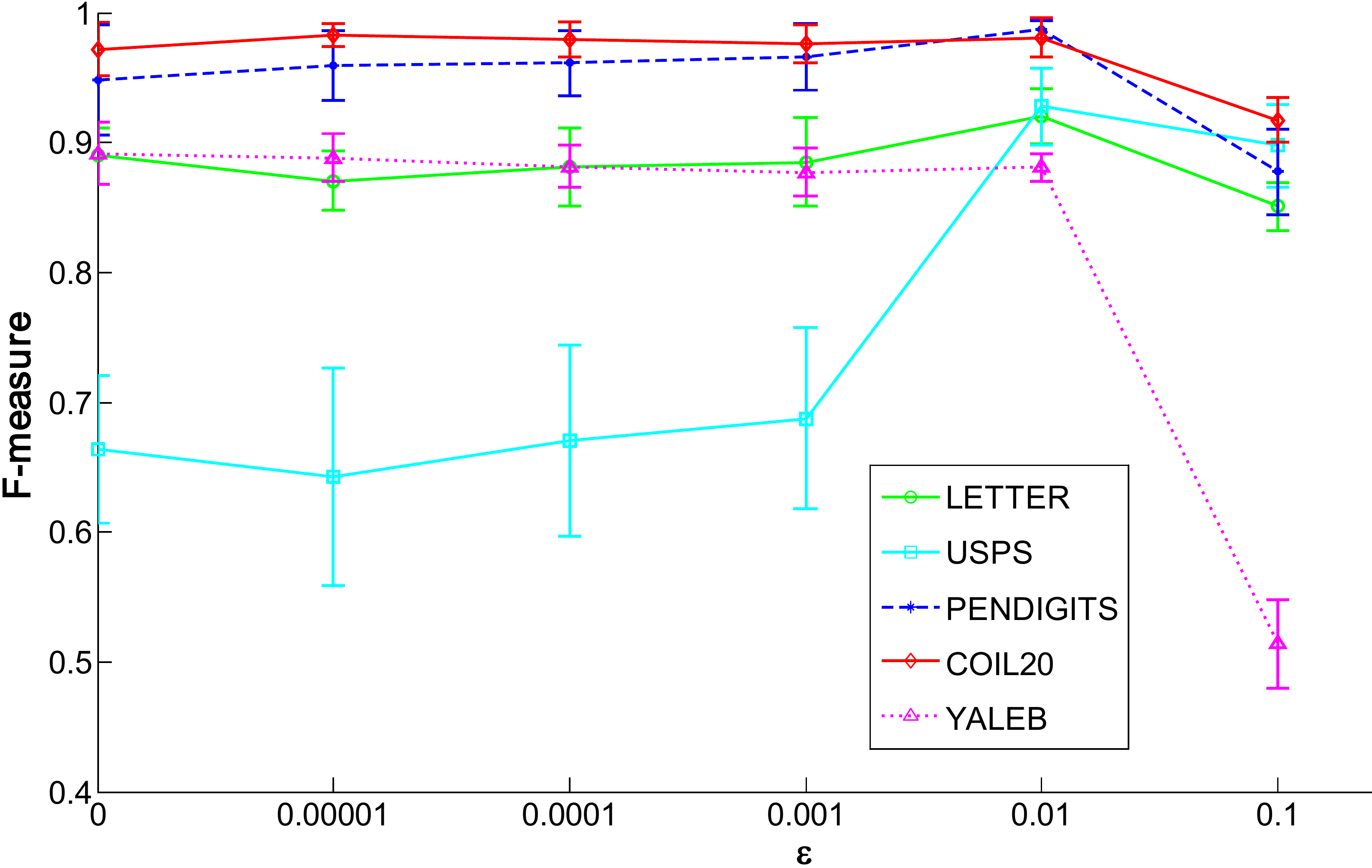}%
\label{fig_second_case}}
\caption{Parameter $\varepsilon$ sensitivity experiment. (a) shows the influence of $\varepsilon$ on classification performance in closed set scenario, while (b) shows this influence in open set scenario.}
\label{fig_sim}
\end{figure*}

\subsection{New class discovery}
In this subsection, we show the new class discovery function under our CD-OSR framework. Unlike the existing methods infer the unknown classes depending on accurately modeling for the known classes, the CD-OSR is able to provide explicit modeling for unknown classes appearing in testing. Thus it can discover new classes. Furthermore, instead of the two-step manner in \cite{Bendale2015Towards,Shu2018Unseen}, CD-OSR adopts a jointly solving process to handles both open set recognition and new class discovery.
As mentioned above, the true labels of unknown classes are unknown, making it impossible to further aggregate the newly generated subclasses. Therefore, each new class will inherently have only one subclass. In other words, these newly discovered classes are just at subclass level. Fortunately, we can still roughly estimate the number of real unknown classes based on the number of subclasses of known classes. This can be used as a prior for the other clustering algorithms (such as K-means, etc.) to further discover the real classes among the reject instances. Concretely, we have
\begin{equation}
\Delta = [\frac{|S_{\text{unknown}}|}{|S_{\text{known}}/(J-1)|} + 0.5],
\end{equation}
where $|S_\text{unknown}|$ denotes the number of subclasses corresponding to unknown classes, $|S_{\text{known}}|$ denotes the number of subclasses of known classes, and $J-1$ here represents the number of known classes. Note that this is just a relatively rough estimate. Actually, a more realistic operation is that we can construct a candidate set according to (14) for other clustering algorithms to quickly determine a more accurate estimate.

Furthermore, Table 1 and 2 respectively report the new class discovery function for USPS and PENDIGITS datasets under CD-OSR framework. Each table has three columns, where the first column denotes the corresponding group data (known classes and testing set), the second one indicates the number of the subclasses of the corresponding group, and the third one represents the proportions of the corresponding subclasses in their group.
\begin{table}[!h]
\renewcommand\arraystretch{1.3}
\tabcolsep 1.3pt
\caption{NEW CLASS DISCOVERY ON USPS DATASET}
\vspace*{-20pt}
\begin{center}
\def\temptablewidth{0.48\textwidth}
{\rule{\temptablewidth}{1pt}}
\begin{tabular*}{\temptablewidth}{@{\extracolsep{\fill}}ccccccccc}
\bf Group  & $\sharp$ Subclass & \multicolumn{7}{c}{Proportion of the corresponding subclass \% }\\
                                  \hline
    Class1 ('2') & 1  & $S_{13}$\\
                 &    & 98.67        \\
    Class2 ('9') & 4  & $S_1$ & $S_{24}$ &$S_{26} $ & $S_{27}$ \\
                 &    & 21.54 & 8.31 & 67.08 & 2.15     \\
    Class3 ('1') & 4  & $S_7$ & $S_{16}$ &$S_{18} $ & $S_{21}$ \\
                 &    & 7.54  & 37.01 & 3.21 & 51.12   \\
    Class4 ('6') & 3  & $S_{14}$ &$S_{15}$ & $S_{17}$ \\
                 &    & 26.05 & 50.30 & 20.96   \\
    Class5 ('3') & 7  & $S_{2}$ & $S_{3}$ &$S_{4}$ & $S_{5}$ & $S_{8}$ &$S_{9}$ & $S_{11}$\\
                 &    & 22.55 & 9.79 & 18.00 & 7.74 & 19.59& 2.73& 18.00     \\
   \hline
   Testing-Set   & 33 & \multicolumn{3}{c}{Known subclasses ($\sharp$: 19)} & \multicolumn{4}{c}{New subclasses ($\sharp$: 14)}\\
                 &    & \multicolumn{3}{c}{55.23} & \multicolumn{4}{c}{44.77}
\end{tabular*}
{\rule{\temptablewidth}{1pt}}

\emph{There are five classes in training set while the testing set has all the classes (5 known classes and 5 unknown classes). The table gives the estimates of mixture proportions and number of subclasses in each group under CD-OSR framework.}
\end{center}
\end{table}

For USPS shown in Table 1, we randomly select 5 classes (the real classes in brackets) as the known classes for training, while the testing set has all of the classes (5 known classes and 5 unknown classes). According to (14), we can obtain the rough estimate
\begin{equation}
\Delta = [\frac{|S_{\text{unknown}}|}{|S_{\text{known}}|/(J-1)} + 0.5] = [\frac{14}{19/5} + 0.5] = 4,
\end{equation}
where the estimated number of unknown classes $\Delta$ approaches the true number of unknown classes. Actually, we may obtain more accurate estimate if the number of subclasses for the corresponding classes are relatively uniform.

Similar to the operation of USPS, we also randomly choose 5 classes as the known classes for training, while testing set owns all the classes (5 known classes and 5 unknown classes). Table 2 reports the specific results, where we can obtain the similar conclusion described above. Moreover, we can discover the internal distribution corresponding to each known class at the subclass level, which can be seen as a by-product of our approach. For example, the distribution of instances corresponding to class 1 ('2') is very concentrated, where almost all the instances are clustered in one subclass $S_{13}$. In contrast, the instances' distribution of class 5 ('3') is relatively scattered, where most of the instances are scattered in 7 subclasses, as shown in Table 1.
\begin{table*}[ht]
\renewcommand\arraystretch{1.2}
\tabcolsep 0pt
\caption{NEW CLASS DISCOVERY ON PENDIGITS DATASET}
\vspace*{-18pt}
\begin{center}
\def\temptablewidth{0.95\textwidth}
{\rule{\temptablewidth}{1pt}}
\begin{tabular*}{\temptablewidth}{@{\extracolsep{\fill}}ccccccccccccccccc}
\bf Group  & $\sharp$ Subclass & \multicolumn{15}{c}{Proportion of the corresponding subclass \%}\\
                                  \hline
    Class1 ('4') & 7  & $S_5$ & $S_8$&$S_{16}$ & $S_{33}$& $S_{34}$& $S_{46}$&$S_{49}$\\
           &    &5.25   &5.39  &52.77    &12.68    &19.97 &1.02 &2.48      \\
    Class2 ('2') & 5  & $S_4$ & $S_{13}$ &$S_{41} $ & $S_{42}$ & $S_{44}$\\
                &    & 5.25 & 69.83 & 3.94 & 5.98 & 14.29     \\
    Class3 ('1') & 11  & $S_7$ & $S_{18}$ &$S_{20} $ & $S_{21}$ & $S_{22}$ &$S_{26}$ & $S_{35}$ &$S_{36}$ &$S_{37}$ &$S_{38}$ & $S_{39}$\\
                &    & 25.22 & 2.92 & 1.60 & 5.54 & 9.77 & 4.08& 13.85 & 33.38 & 1.17& 1.17& 1.02      \\
    Class4 ('9') & 15  & $S_1$ & $S_{2}$ &$S_{6} $ & $S_{9}$ & $S_{10}$ &$S_{11}$ & $S_{12}$ &$S_{24}$ &$S_{25}$ &$S_{27}$ & $S_{28}$ & $S_{29}$ & $S_{30}$ &$S_{31}$ & $S_{76}$\\
                &    & 2.53 & 7.74 & 6.00 & 2.05 & 13.11 & 18.96& 7.90 & 2.69 & 8.06& 7.27& 5.21& 8.06& 1.74& 4.90& 1.42      \\
    Class5 ('6') & 5  & $S_{14}$ & $S_{15}$ &$S_{40} $ & $S_{43}$ & $S_{48}$\\
                &    & 43.85 & 35.96 & 7.57 & 11.04 & 1.42     \\
   \hline
   Testing-Set & 75 & \multicolumn{7}{c}{Known subclasses ($\sharp$: 43)} & \multicolumn{8}{c}{New subclasses ($\sharp$: 32)}\\
               &    & \multicolumn{7}{c}{50.42} & \multicolumn{8}{c}{49.58}
\end{tabular*}
{\rule{\temptablewidth}{1pt}}

\emph{There are five classes in training set while the testing set has all the classes (5 known classes and 5 unknown classes). The table gives the estimates of mixture proportions and number of subclasses in each group under CD-OSR framework.}
\end{center}
\end{table*}

\section{Conclusion}
The main contribution of this paper is to present a collective/batch decision strategy for open set recognition with an aim to extend existing open set recognition for new class discovery while considering correlations among the testing instances. To achieve this goal, we adapt HDP with slight modification to addressing the OSR problem, leading an initial solution towards collective decision in OSR. What needs to be highlighted is that our CD-OSR does not overly depend on the training data and can achieve adaptive change as the data changes. More precisely, CD-OSR can provide explicit modeling for unknown classes appearing in testing. This naturally leads to the new class discovery function, even though it is just at the subclass level. Furthermore, unlike the existing methods dealing with the OSR problem from the discriminative model perspective, the CD-OSR actually addresses this problem from the generative model perspective due to the use of HDP. Finally, the experimental results on a set of benchmark datasets indicate the validity of our learning framework.

Besides, it should be noted that modeling unknown classes only performs in the testing phase of our CD-OSR, whilst no available knowledge from unknown classes is utilized during the training phase. This seems to have the flavor of lazy learning to some extent. Thus the co-clustering process (testing process) will be repeated when other batch testing data arrives, resulting in higher computational overhead. Therefore, overcoming this limitation will be a promising research direction in the future. Furthermore, since the CD-OSR currently does not make full use of the discriminative information from the known class labels, embedding this kind of information more effectively will be also worth further exploring. In addition, replacing the Gibbs sampler with scalable deterministic inference techniques is a promising direction as well in the future work. In conclusion, the CD-OSR is just as a conceptual proof for open set recognition towards collective decision at present. Therefore, the more effective collective decision methods for OSR are worth further exploring in the future work.

\ifCLASSOPTIONcompsoc
  \section*{Acknowledgments}
\else
  \section*{Acknowledgment}
\fi

The authors would like to thank the support from NSFC under Grant No. 61672281, the Key Program of NSFC under Grant No. 61732006 and the Postgraduate Research \& Practice Innovation Program of Jiangsu Province under Grant No. KYCX18\_0306.

\ifCLASSOPTIONcaptionsoff
  \newpage
\fi



%

%
%
\bibliographystyle{ieeetr}
\bibliography{mybibfile}

\end{document}